\documentclass[11pt]{article}
\usepackage[margin=1in]{geometry}
\usepackage{more_styles}

\usepackage[utf8]{inputenc} 
\usepackage[T1]{fontenc}    
\usepackage{hyperref}       
\usepackage{url}            
\usepackage{booktabs}       
\usepackage{amsfonts}       
\usepackage{nicefrac}       
\usepackage{microtype}      
\usepackage{xcolor}         

\newenvironment{phase}[1][htb]
  {
   \begin{algorithm}[#1]%
  }{\end{algorithm}}

\title{Global Convergence of Federated Learning for Mixed Regression 
}

\author{
Lili Su \\
Electrical and Computer Engineering \\
Northeastern University \\ 
{l.su@northeastern.edu}
\and 
Jiaming Xu\\
The Fuqua School of Business\\
Duke University \\
{jiaming.xu868@duke.edu}
\and 
Pengkun Yang \\
Center for Statistical Science \\ 
Tsinghua University \\
{yangpengkun@tsinghua.edu.cn} 
}

\date{\today}

\begin{document}

\maketitle

\begin{abstract}

This paper studies the problem of model training under Federated Learning when clients exhibit cluster structure. 
We contextualize this problem in mixed regression, where each client has limited local data generated from one of $k$ unknown regression models. We design an algorithm that achieves global convergence from any initialization, and works even when local data volume is highly unbalanced  -- there could exist clients that contain $O(1)$ data points only. Our algorithm first runs moment descent on a few anchor clients (each with $\tilde{\Omega}(k)$ data points) to obtain coarse model estimates. Then each client alternately estimates its cluster labels  and refines the model estimates based on FedAvg or FedProx. A key innovation in our analysis is a uniform estimate on the clustering errors, which we prove by bounding the VC dimension of general polynomial concept classes based on the theory of algebraic geometry.

\end{abstract}

\section{Introduction}
\label{sec: intro}
Federated learning (FL) \cite{mcmahan2017communication} 
enables a massive number of 
clients to  collaboratively 
train models without disclosing raw data. Heterogeneity in clients renders local data non-IID and highly unbalanced. 
For example, smartphone users have different preferences in article categories (e.g. politics, sports or entertainment) and have a wide range of reading frequencies. In fact, distribution of the local dataset sizes 
is often heavy-tailed \cite{dooms2013movietweetings,FACebook2020,feuerverger2012statistical}.   

Existing methods to deal with data heterogeneity 
can be roughly classified into three categories:  
a common model, fully personalized models, 
and clustered personalized models; 
see Section \ref{sec: related work} for detailed discussion.  
Using a common model to serve highly heterogeneous clients has fundamental drawbacks;  recent work \cite{su2021achieving} rigorously quantified the heterogeneity level that the common model can tolerate with. 
Fully personalized models  \cite{smith2017federated,marfoq2021federated} often do not come with performance guarantees as the underlying optimization problem is generally hard to solve. 
In this work, we focus on clustered personalized models \cite{sattler2020clustered},  
i.e., clients within the same cluster share the same underlying model and clients across clusters have relatively different underlying models. The main challenge is that the cluster identities of the clients are unknown. We target at designing algorithms that 
simultaneously 
learn the clusters and train models for each cluster. 

A handful of simultaneous clustering and training algorithms 
have been proposed \cite{ mansour2020three,xie2021multi,li2021federatedsoftlearning,ghosh2020efficient}, mostly of which are heuristic and lack of convergence guarantees \cite{mansour2020three,xie2021multi,li2021federatedsoftlearning}.  Towards formal assurance, \cite{ghosh2020efficient} studied this problem through the lens of statistical learning yet postulated a number of strong assumptions such as the initial model estimates are very close to the true ones, linear models, strong convexity, balanced and high-volume of local data,
and sample splitting across iterations. 
Their numerical results \cite{ghosh2020efficient} suggested that sufficiently many random initializations would lead to at least one good realization satisfying the required closeness assumption. However, the necessary number of random initializations scales exponentially in both the input dimension and the number of clusters. Besides, in practice, it is hard to recognize which initialization is good and to winnow out high quality trained models. 
%
In this work, following \cite{ghosh2020efficient}, we adopt a statistical learning setup. In particular, we contextualize this problem as the canonical mixed regression, where each client has a set of local data generated from one of $k$ unknown regression models. Departing from standard mixed regression, in which each client keeps one data point only \cite{mccullagh2019generalized,li2018learning}, in our problem the sizes of local datasets can vary significantly across clients. We only make a mild assumption that there exist a few anchor clients (each with $\tilde{\Omega}(k)$ data points) 
and sufficiently many clients with at least two data points. 
Similar mixed regression setup with data heterogeneity has been considered in \cite{kong2020meta}  in a different context of meta-learning; 
the focus there is on exploring structural similarities among a large number of tasks in centralized learning. 
Please refer to~\prettyref{rmk:comparision_warm_start} for more detailed technical  comparisons. 


\paragraph{Contributions}
The main contributions of this work are summarized as follows:
\begin{itemize}
\item We design a two-phase federated learning algorithm to learn clustered personalized models in the context of mixed regression problems. In Phase 1, the parameter server runs a federated moment descent on a few anchor clients to obtain coarse model estimates based on subspace estimation. In each global round of Phase 2, each client alternately estimates its cluster label  and refines the model estimates based on FedAvg or FedProx.
The algorithm works even when local data volume is highly unbalanced  -- there could exist clients that contain $O(1)$ data points only. 
\item We theoretically prove the global convergence of our algorithm from any initialization. The proof is built upon two key ingredients: 1) We develop a novel eigengap-free bound to control the projection errors in subspace estimation; 2)
To deal with the sophisticated interdependence between the two phases and across iterations, 
we develop a novel uniform estimate on the clustering errors, which we derive by bounding the VC dimension of general polynomial concept classes based on the theory of algebraic geometry.
 Our analysis reveals that the final estimation error is dominated by the uniform deviation of the clustering errors, which is largely overlooked by the previous work. 
\end{itemize}

\section{Related Work}
\label{sec: related work}
FedAvg \cite{mcmahan2017communication} is a widely adopted FL algorithm due to its simplicity and low communication cost. However, severe data heterogeneity could lead to unstable training trajectories and land in suboptimal models \cite{li2020federated,zhao2018federated,karimireddy2020scaffold}. 

\noindent{\bf A Common Model:}
To limit the negative impacts of data heterogeneity on the obtained common model, a variety of techniques based on variance reduction \cite{li2020federated,li2019rsa,karimireddy2020scaffold} and normalization \cite{wang2020tackling} have been introduced. Their convergence results mostly are derived under strong technical assumptions such as bounded gradient and/or bounded Hessian similarity which do not hold 
when the underlying truth in the data generation is taken into account \cite{li2020federated,li2019rsa,karimireddy2020scaffold}.
In fact, none of them strictly outperform others in different instances of data heterogeneity \cite{li2021federated}.
Besides, the generalization errors of the common model with respect to local data  
are mostly overlooked except for a recent work \cite{su2021achieving}, which shows that the common model can tolerate a moderate level of model heterogeneity.

\noindent{\bf Fully Personalized Models: }
\cite{smith2017federated} proposed Federated Multi-Task Learning (MTL) wherein different models are learned for each of the massive population of clients \cite{smith2017federated,marfoq2021federated}. Though conceptually 
attractive, the convergence behaviors of Federated MTL is far from well-understood because 
the objective is not jointly convex in the model parameters and the model relationships \cite{smith2017federated,argyriou2008convex,ando2005framework}. 
Specifically, \cite{smith2017federated} focused on solving the subproblem of updating the model parameters only. 
Even in the centralized setting, convergence is only shown under rather restricted assumptions such as equal dataset sizes for different tasks \cite{ando2005framework} and small number of common features 
\cite{argyriou2008convex}. 
Moreover,  the average excess error 
rather than the error of individual tasks is shown to decay with the dominating term $O(1/\sqrt{n})$, where $n$ is the homogeneous local dataset size \cite{ando2005framework}. 
Despite recent progress \cite{tripuraneni2021provable,du2020few}, their results are mainly for linear representation learning and for balanced local data.  
Parallel to Federated MTL, model personalization is also studied under the Model-Agnostic Meta-Learning (MAML) framework \cite{fallah2020personalized,jiang2019improving} where the global objective is modified to account for the cost of fine-tuning a global model at individual clients. However, they focused on studying the convergence in training errors only-- no characterization of the generalization error is given.

\noindent{\bf Clustered Personalized Models:} 
Clustered Federated Learning (CFL) \cite{sattler2020clustered,ghosh2020efficient,ghosh2019robust,xie2021multi,mansour2020three,li2021federatedsoftlearning} can be viewed as a special case of Federated MTL where tasks across clients form cluster structures. 
In addition to the algorithms mentioned in Section \ref{sec: intro}, i.e., which  simultaneously learn clusters and models for each cluster, 
other attempts have been made to integrate clustering with model training.  
\cite{sattler2020clustered} hierarchically clustered the clients in a post-processing fashion. 
To recover the $k$ clusters, $\Omega(k)$  empirical risk minimization problems need to be solved sequentially -- which is time-consuming. 
\cite{ghosh2019robust} proposed a modular algorithm that contains one-shot clustering stage,  
 followed by $k$ individual adversary-resilient model training. 
Their algorithm scales poorly in the input dimension, and requires local datasets to be balanced and sufficiently large (i.e., $n\ge d^2$). Moreover, 
each client sequentially solves two empirical risk minimization problems. 
To utilize information across different clusters, \cite{li2021federatedsoftlearning} proposed soft clustering. Unfortunately, the proposed algorithm is only tested on the simple MNIST and Fashion-MNIST datasets, and no theoretical justifications are given.

\section{Problem Formulation}
\label{sec: problem formulation}
%
%




The FL system consists of a parameter server (PS) and $M$ clients. 
Each client $i\in [M]$  keeps a dataset $\calD_i=\sth{\pth{x_{ij}, y_{ij}}}_{j=1}^{n_i}$ that are generated from one of $k$ unknown regression models. 
Let 
$N = \sum_{i=1}^M n_i$. 
The local datasets are highly unbalanced with varying $n_i$ across $i$.
If $n_i = \tilde{\Omega}(k)$, we refer to client $i$ as \emph{anchor} client, which corresponds to active user in practice.  Anchor clients play a crucial rule in our algorithm design. We consider the challenging yet practical scenario wherein a non-anchor client may have $O(1)$ data points only. 

 
We adopt a canonical mixture model setup: 
For each client $i\in[M]$,
\begin{equation}
\label{eq:model}
y_i=\ph{\bx_i} \theta_{z_i}^*+\zeta_i,
\end{equation}
where
$z_i\in[k]$ is the \emph{hidden} local cluster label, 
$\theta_1^*, \cdots, \theta_k^*$ are the true models of the clusters,  
$\ph{\bx_i} \in \reals^{n_i\times d}$ is the 
feature matrix with rows given by $\ph{x_{ij}}$, 
$y_i=(y_{ij}) \in \reals^{n_i}$ 
is the response vector, and $\zeta_i=(\zeta_{ij}) \in \reals^{n_i}$ is the 
noise vector. Note that the feature map $\phi$ can be non-linear (e.g.\,polynomials). 
The cluster label of client $i$ is randomly generated from one of the $k$ components from some unknown $p=(p_1,\dots,p_k)$ in probability simplex $ \psimplex^{k-1}$.
That is, $\prob{z_i=j}=p_j$ for $j\in [k]$. 
In addition, 
$\Norm{\theta_{j}^*}\le R$ for each component.
The feature covariate $\ph{x_{ij}}$ is independent and sub-Gaussian 
$\alpha I_d\preceq \Expect[\ph{x_{ij}}\ph{x_{ij}}^\top] \preceq \beta I_d$.
We assume that the covariance matrix is identical within the same cluster but may vary across different clusters,
\ie, $ \Expect[\ph{x_{ij}}\ph{x_{ij}}^\top] = \Sigma_{j}$ if $z_i=j.$
The noise $\zeta_{ij}$ is independent and sub-Gaussian with $\Expect[\zeta_{ij}]=0$ and $\Expect[\zeta_{ij}^2]\le \sigma^2$. 

Our formulation accommodates  heterogeneity in feature covariates, underlying local  models,  
and observation noises \cite{kairouz2019advances}. 
For the identifiability of the true cluster models $\theta_j^*$'s, 
we assume a minimum proportion and a pairwise separation of clusters. 
Formally, let $\Delta=\min_{j\ne j'}\Norm{\theta_j^*-\theta_{j'}^*}$ and $p_{\min}=\min_{j \in [k]} p_j$.
For ease of presentation, we assume the parameters $\alpha, \beta = \Theta(1)$, $\sigma/\Delta=O(1),$ and $R/\Delta=O(1)$, while our main results can be extended to show more explicit dependencies on these parameters with careful bookkeeping calculations.

\noindent{\bf Notations:}
Let $[n]\triangleq \{1,\dots,n\}$. 
For two sets $A$ and $B$, let $A \ominus B$ denote the symmetric difference $(A- B)\cup (B- A)$.
We use standard asymptotic notation: for two positive sequences $\{a_n\}$ and $\{b_n\}$, we write $a_n = O(b_n)$ (or $a_n \lesssim b_n$) if $a_n \le C b_n$ for some constant $C$ and sufficiently large  $n$; $a_n = \Omega(b_n)$ (or $a_n \gtrsim b_n$) if $b_n = O(a_n)$; $a_n = \Theta(b_n)$ (or $a_n \asymp b_n$) if $a_n = O(b_n)$ and $a_n=\Omega(b_n)$; 
Poly-logarithmic factors are hidden in $\tilde \Omega $. Given a matrix $A \in \reals^{n\times d}$, let $A=\sum_{i=1}^r \sigma_i u_i v_i^\top$ denote its singular value decomposition, where $r=\min\{n,d\}$, $\sigma_1 \ge \cdots \ge \sigma_r\ge 0$ are the singular values, and $u_i$ ($v_i$) are the corresponding left (right) singular vectors. We call $U=[u_1, u_2, \ldots, u_k]$ as the top-$k$ left singular matrix of $A.$ Let $\spn(U)=\spn\{u_1, \ldots, u_k\}$ denote the $k$-dimensional subspace spanned by $\{u_1, \ldots, u_k\}.$

\section{Main Results}
\label{sec: main results}
%

We propose a two-phase FL algorithm that enables clients to learn the 
model parameters ${\theta}_1^*,\dots,{\theta}_k^*$ and their clusters simultaneously. 
\begin{itemize}
\item[Phase 1:] \emph{(Coarse estimation via FedMD.)}
Run the federated moment descent algorithm to obtain
coarse estimates of model parameters $\theta_i^*$'s. 
\item[Phase 2:] \emph{(Fine-tuning via iterative FedX+clustering.)} 
In each round, each client first estimates its cluster label 
and then refines its local model estimate 
via either FedAvg or FedProx (which we refer to as FedX) \cite{mcmahan2017communication,li2020federated}. 
\end{itemize}
The details of FedMD and FedX+clustering are specified in \prettyref{phase:moment_descent} and \prettyref{phase:iterative-FL}, respectively.



\subsection{Federated moment descent}
\label{sec:start}
The key idea of the first phase of our algorithm is to leverage the existence of anchor clients. 
Specifically, the PS chooses a set $H$ of $n_H$ anchor clients uniformly at random. Each selected anchor client $i \in H$ maintains a sequence of estimators $\{ \theta_{i,t}\}$ 
that approaches  $\theta^*_{z_i}$, achieving
$\Norm{{\theta}_{i,t} - \theta^*_{z_i}} \le \epsilon \Delta$
for some small constant $\epsilon>0$ when $t$ is sufficiently large. 

At high-level, we hope to have $\theta_{i,t}$ move along
a well calibrated direction $r_{i,t}$ that decreases the residual estimation error $\Norm{\Sigma_{z_i}(\theta^*_{z_i}-\theta_{i,t})}$, \ie, the variance of the residual
$\iprod{\phi(x_{ij})}{\theta^*-\theta_{i,t}}$. As such, we like to choose $r_{i,t}$ to be positively correlated with $\Sigma_{z_i}(\theta^*_{z_i}-\theta_{i,t})$. However, to estimate $\Sigma_{z_i}(\theta^*_{z_i}-\theta_{i,t})$ solely based on the local data of anchor client $i$, it requires  $n_i=\tilde{\Omega}(d)$, which is unaffordable in typical FL systems with high model dimension $d$ and limited local data. To resolve the curse of dimensionality, we decompose the estimation task at each chosen anchor client into two subtasks: we first estimate a $k$-dimensional subspace that $\Sigma_{z_i}(\theta^*_{z_i}-\theta_{i,t})$ lies in by pooling local datasets across sufficiently many non-anchor clients; then we project the local data of anchor client $i$ onto the estimated subspace and reduce the estimation problem from $d$-dimension to $k$-dimension. 

The precise description of our \prettyref{phase:moment_descent} procedure is given as below. 
For ease of notation, let $\varepsilon(x,y,\theta)\triangleq(y-\iprod{\ph{x}}{\theta})\ph{x}$.

\begin{phase}[htb]
\caption{Federated Moment Descent (FedMD)}\label{phase:moment_descent}
\textbf{Input:} $n_H, k, m, \ell, T, T_1, T_2 \in \naturals$, $\alpha, \beta, \epsilon, \Delta \in \reals$, $\theta_0 \in \reals^d$ with $\norm{\theta_0} \le R$ \\
\textbf{Output:} $\hat{\theta}_1, \ldots, \hat{\theta}_k$

\vskip 0.5\baselineskip

PS chooses a set $H$ of $n_H$ anchor clients uniformly at random\; 

\For{each anchor client $i\in H$}
{$\theta_{i,0}\gets \theta_0$\;}

\For{$t=0,1, \ldots, T-1$}
{
PS selects a set $\calS_t$ of $m$ clients from $[M]\setminus \pth{H\cup \pth{\cup_{\tau=0}^{t-1} \calS_{\tau}}}$\; 
PS broadcasts $\{\theta_{i,t}, i \in H\}$ to all clients $i'$ in $\calS_t$;  \tcc*[h]{\color{blue} where $\cup_{\tau=0}^{-1} \calS_{\tau} = \emptyset$}\;   
PS calls {\sf federated-orthogonal-iteration} ($\calS_t$, $\{ 
\varepsilon(x_{i'1},y_{i'1},\theta_{i,t}),\varepsilon(x_{i'2},y_{i'2},\theta_{i,t})
\}_{i' \in \calS_t}$, $k$, $T_1$) to output $\hat{U}_{i,t}$ for each anchor client $i \in H$;
\quad   \tcc*[h]{\color{blue}described in~\prettyref{alg:OI}}   

PS sends $\hat{U}_{i,t}$ to each anchor client $i \in H$;

Each anchor client $i$ calls {\sf power-iteration}($A_{i,t}A^\top_{i,t}$, $T_2$) to output $(\hat{\beta}_{i,t}, \hat{\sigma}^2_{i,t})$ with $A_{i,t}$ defined in~\prettyref{eq:A_i_t}; 

\uIf{$\hat{\sigma}_{i,t}> \epsilon \Delta$}
{
$\theta_{i, t+1} \gets \theta_{i,t} + r_{i,t} \eta_{i,t}$ and reports $\theta_{i, t+1}$, where $r_{i,t}=\hat{U}_{i,t} \hat{\beta}_{i,t}$ and $
\eta_{i,t} =  \alpha \hat{\sigma}_{i,t}/(2\beta^2)$\; 

}\Else{$\theta_{i,t}\gets \theta_{i,t+1}$ and reports $\theta_{i, t+1}$\;
}

}

PS computes the pairwise distance $\norm{\theta_{i,T}-\theta_{i',T}}$ for every pair of anchor clients $i, i' \in H$, assigns them in the same cluster when the pairwise distance is smaller than $\Delta/2$, and outputs $\hat{\theta}_j$ to be the center of the estimated $j$-th cluster for $j \in [k]$.

\end{phase}






In Step 9, PS estimates the subspace that the residual estimation errors $\{\Sigma_j (\theta_j^*-\theta_{i,t})\}_{j=1}^k$
lie in, in collaboration with clients in $\calS_t$.
 In particular,   for each anchor client $i \in H$,  define $$
Y_{i,t}= \frac{1}{m} \sum_{i' \in \calS_t} \varepsilon(x_{i'1},y_{i'1},\theta_{i,t})
\varepsilon(x_{i'2},y_{i'2},\theta_{i,t})^\top
$$
We approximate the subspace spanned by $\{\Sigma_j (\theta_j^*-\theta_{i,t})\}_{j=1}^k$ via that spanned by the top-$k$ left singular vectors of $Y_{i,t}.$ To compute the latter, we adopt the following multi-dimensional generalization of the power method, known as {\sf orthogonal-iteration}~\cite[Section 8.2.4]{Golub-VanLoan2013}. 
In general, given a symmetric matrix $Y \in \reals^{d \times d}$,
the orthogonal iteration generates a sequence of matrices $Q_t \in \reals^{d \times k}$ as follows: $Q_0 \in \reals^{d\times k}$ is initialized as a random orthogonal matrix $Q_0^\top Q_0=I$ and $Y Q_t = Q_{t+1} R_{t+1}$ with QR factorization. When $t$ is large,
$Q_t$ approximates the top-$k$ left singualr matrix of $Y$, provided the existence of an eigen-gap $\lambda_k>\lambda_{k+1}$.
When $k=1$, this is just the the {\sf power-iteration} and we can further approximate the leading eigenvalue of $Y$
by the Raleigh quotient $Q_t^\top Y Q_t.$ When $Y$ is asymmetric, by running 
the orthogonal iteration on $YY^\top$, we can compute the top-$k$ left singular matrix of $Y$.
In our setting, the orthogonal iteration can be implemented in a distributed manner in FL systems as shown in~\prettyref{alg:OI} in the Appendix. 

In Step 11, each anchor client $i$ estimates the residual error $\Sigma_{z_i} (\theta_{z_i}^*-\theta_{i,t})$ by projecting $\varepsilon(x_{ij},y_{ij},\theta_{i,t})$ onto the previously estimated subspace, that is,  $\hat{U}_{i,t}^\top \varepsilon(x_{ij},y_{ij},\theta_{i,t})$. 
This reduces the estimation from 
$d$-dimension to $k$-dimension and hence
$\tilde{\Omega}(k)$ local data points suffice. 
Specifically, define
 \begin{align}
A_{i,t}= \frac{1}{\ell} \sum_{j\in \calD_{i,t}}  \left( \hat{U}_{i,t}^\top \varepsilon(x_{ij},y_{ij},\theta_{i,t}) \right)
\left( \hat{U}_{i,t}^\top \varepsilon(\tilde{x}_{ij},\tilde{y}_{ij},\theta_{i,t}) \right)^\top , \label{eq:A_i_t}
\end{align}
 where $\calD_{i,t}$ consists of $2\ell$ local data points $(x_{ij}, y_{ij})$ and $(\tilde{x}_{ij}, \tilde{y}_{ij})$ freshly drawn from $\calD_i$ at iteration $t$. Client $i$ runs the {\sf power-iteration} to output $\hat{\beta}_{i,t}$ and $\hat{\sigma}^2_{i,t}$ as approximations of the leading left singular vector
 and singular value of $ A_{i,t} $, 
Then anchor client $i$ updates a new estimate $\theta_{i,t+1}$ by moving along the direction of the estimated residual error $r_{i,t}$ with an appropriately chosen step size $\eta_{i,t}$. 

We show that 
$\theta_{i,T}$
 is  close to $\theta^*_{z_i}$ for every anchor client $i \in H$ 
 and 
the outputs $\hat{\theta}_j$ are close to $\theta^*_j$
up to a permutation of cluster indices.
\begin{theorem}\label{thm:warm_start}
Let $\epsilon \in (0,1/4)$ be a  small but fixed constant. Suppose  that
\begin{align}
 m \ge p_{\min}^{-2}  \tilde{\Omega}(d) ,   \; \ell =\tilde{\Omega}( k),    \; T=\Omega\left(1 \right),  \; T_1=\Omega\left( k \log (Nd) \right), \; T_2=\Omega\left( \log(Nd) \right). \label{eq:warm_start_cond}
\end{align}
With probability at least $1-O(N^{-9})$, for all initialization $\theta_0$ with $\norm{\theta_0}\le R,$
\begin{align}
\sup_{i \in H} \norm{\theta_{i,T} - \theta^*_{z_i}} \le \epsilon \Delta.  \label{eq:warm_start_error}
\end{align}
Furthermore, when $n_H \ge \log N/p_{\min}, $  with probability at least $1-O(N^{-9})$, 
\begin{align}
d(\hat{\theta}, \theta^* )\le \epsilon \Delta, \label{eq:warm_start_cluster}
\end{align}
where 
\begin{equation}
 \label{eq:distance}    
 d(\hat\theta, \theta^*)
 \triangleq \min_\pi \max_{j\in[k]}\norm{\hat\theta_{\pi(j)}-\theta_j^*}, 
 ~~~ \text{where} ~ \pi ~ \text{is permutation over $[k]$.}
 \end{equation}
\end{theorem}

Note that in~\prettyref{eq:distance} we take a minimization over permutation, as
 the cluster indices are unidentifiable. 
 
\prettyref{phase:moment_descent} 
uses fresh data at every iteration.
In total we need
$p_{\min}^{-2}\tilde{\Omega}(d)$ clients with at least two data points and $\tilde{\Omega}(1/p_{\min})$ anchor clients. This requirement is relatively mild,  as 
typical FL systems 
 have a large number of clients with $O(1)$ data points and a few anchor clients with moderate data volume.

 We defer the detailed proof of \prettyref{thm:warm_start} to~\prettyref{app:warm_start}.
A key step in our proof is to show 
the residual estimation errors $\{ \Sigma_j(\theta_j^*- \theta_{i,t}) \}_{j=1}^k$ approximately lie in $\spn(\hat{U}_{i,t}).$
Unfortunately, the eigengap of $Y_{i,t}$ could be small, especially when $\theta_{i,t}$ gets close to  
$\theta^*_{z_i}$; and hence the standard Davis-Kahan theorem~\cite{DavisKahan70} cannot be be applied. This issue is
further exaggerated by the fact that the convergence rate of the orthogonal iteration also crucially depends on the eigengaps~\cite{Golub-VanLoan2013}. For these reasons, $\spn(\hat{U}_{i,t})$ may not be close to $\spn\{ \Sigma_j(\theta_j^*- \theta_{i,t}) \}_{j=1}^k$ at all. To resolve this issue, we develop a novel gap-free bound to show that
projection errors $ \hat{U}_{i,t}^\top \Sigma_j(\theta_j^*- \theta_{i,t})$ are small for every $j \in [k]$
(cf.~\prettyref{lmm:DK_Power}).


\begin{remark}[Comparison to previous work]\label{rmk:comparision_warm_start}
Our algorithm is partly inspired by~\cite{li2018learning} which focuses on the noiseless mixed linear regression, but
deviates in a number of crucial aspects. 
 First, our algorithm crucially utilizes the fact that each  client chosen in $S_t$ has at least two data points and hence the space of the singular vectors of $\expect{Y_{i,t}}$ is spanned by 
    $\{\Sigma_j (\theta_j^*-\theta_{i,t})\}_{j=1}^k$. In contrast, \cite{li2018learning} relies on the sophisticated method of moments which only works under the Gaussian features  and requires exponential in $k^2$ many data points. Second, our algorithm further crucially exploits the existence of anchor clients and greatly simplifies the moment descent algorithm in~\cite{li2018learning}.

Our algorithm also bears similarities with the meta-learning algorithm in~\cite{kong2020meta}, which also uses clients collectively
for subspace estimation and anchor clients for estimating cluster centers. However, there are several key differences. First, \cite{kong2020meta} focuses on the centralized setting and relies on one-shot estimation,
    under the additional assumption that the covariance matrix of features across all clusters are identical.  Instead, our moment descent algorithm is iterative, amenable to a distributed implementation in FL systems, and allows for covariance matrices varying across clusters. Second, in the fine-tuning phase, \cite{kong2020meta} uses the centralized least squares to refine the clusters estimated with anchor clients, under the additional assumption that $\Omega(\log k)$ data points for every client. 
    In contrast, as we will show later, we use the FedX+clustering to iteratively cluster clients and
    refine cluster center estimation. 
\end{remark}

\subsection{FedX+clustering}
\label{sec:iterative}
At the end of \prettyref{phase:moment_descent}, only the selected anchor clients in $H$ obtained coarse estimates of their local true models (characterized in \eqref{eq:warm_start_error}). 
In \prettyref{phase:iterative-FL}, both anchor clients in $H$ and 
all the other clients (anchor or not) will 
participate and update their local model estimates. 

\prettyref{phase:iterative-FL} is stated in a generic form for any loss function $L(\theta, \lambda; \calD)$, where $\theta = (\theta_1,\dots,\theta_k)\in \reals^{dk}$ is the cluster parameters, 
$\lambda\in\psimplex^{k-1}$ represents the likelihood of the cluster identity of a client,
and $\calD$ denotes the client's dataset. This generic structure covers the idea of soft clustering \cite{li2021federatedsoftlearning}. 
Note that unlike \prettyref{phase:moment_descent} where each anchor client $i$ only maintains an estimate $\theta_{i,t}$
of its own model, in~\prettyref{phase:iterative-FL}, each client $i$ maintains model estimates
 $\theta_{i\cdot,t}=(\theta_{i1,t},\dots,\theta_{ik,t})$ for all clusters. 
 
\begin{phase}[htb]
\caption{FedX+clustering}\label{phase:iterative-FL}
\textbf{Input:} $\theta=(\theta_1,\dots,\theta_k)$ from the output of \prettyref{phase:moment_descent}, 
$\eta, T'$.\\
\textbf{Output:} {$\hat\theta=(\hat\theta_1,\dots,\hat\theta_k)$}

 PS sets $\theta_T \gets \theta$.
 
\For{$t=T+1,\dots,T+T'$}
{
 PS broadcasts $\theta_{t-1}$ to all clients\;  
 
 Each client $i$ estimates the likelihood of its local cluster label by
\begin{equation}
\label{eq:estimate-label}
\lambda_{i,t}\gets\arg\min_{\lambda\in \psimplex^{k-1}}L(\theta_{t-1},\lambda;\calD_i); 
\end{equation}

Each client $i$ refines its local model based on either FedAvg or FedProx with 
$L_i(\theta)=L(\theta,\lambda_{i,t};\calD_i)$, and reports the updated local parameters $\theta_{i\cdot,t}=(\theta_{i1,t},\dots,\theta_{ik,t})$.
\begin{itemize}
\item FedAvg-based: it runs $s$ steps of local gradient descent:
\begin{align*}
\theta_{i\cdot,t} \gets \calG_i^s (\theta_{t-1}),\quad \text{where } \calG_i(\theta) = \theta- \eta \nabla L_i(\theta)\; 
\end{align*}
\item FedProx-based: it solves the local proximal optimization:
\begin{align*}
\theta_{i\cdot,t} \gets \arg\min_{\theta}L_i(\theta)+\frac{1}{2\eta}\Norm{\theta-\theta_{t-1}}^2\; 
\end{align*}
\end{itemize}

 PS updates the global model as 
$
\theta_t \gets \sum_{i=1}^M w_i \theta_{i\cdot,t},
$
where $w_i=n_i/N$.
}
\end{phase}


In \prettyref{phase:iterative-FL}, the local estimation at each client has a flavor of alternating minimization: It first runs a minimization step to estimate its cluster,  and then runs a FedAvg or FedProx update to refine model estimates.  
To allow the participation of clients with $O(1)$ data points only, at every iteration the clients are allowed to reuse all local data, including those used in the first phase.
Similar alternating update is analyzed in \cite{ghosh2020efficient} yet under the strong assumption that the update in each round is over fresh data with  Gaussian distribution. Moreover, the analysis therein is restricted to the setting where the model refinement at each client is via running a single gradient step, which is barely used in practice but much simpler to analyze than FedAvg or FedProx update. 

In our analysis, we consider the square loss
\[
L(\theta,\lambda;\calD_i)=\frac{1}{2n_i}\sum_{j=1}^k \lambda_j \norm{y_i-\ph{\bx_i}\theta_j}^2. 
\]
In this context, 
\eqref{eq:estimate-label} yields a vertex of the probability simplex $\lambda_{ij,t}=\indc{j=z_{i,t}}$, where
\begin{equation}
\label{eq:hat-z}
z_{i,t} = \arg\min_{j\in[k]}\Norm{ y_i - \ph{\bx_i} \theta_{j,t-1} }.
\end{equation}
The estimate $z_{i,t}$ provides a hard clustering label.
Hence, in each round, only one regression model will be updated per client.

To capture the tradeoff between communication cost and statistical accuracy using FedAvg or FedProx, we introduce the following quantities from \cite{su2021achieving}:
\[
\gamma\triangleq \eta \max_{i\in[M]}\frac{1}{n_i}\norm{\ph{\bx_i}}^2,
\qquad
\kappa  \triangleq  
\begin{cases}
\frac{\gamma s}{1- (1-\gamma)^s} & \text{for FedAvg},\\
1+\gamma & \text{for FedProx}.
\end{cases}
\]
We choose a properly small learning rate $\eta$ such that $\gamma<1$. Here, $\kappa\ge 1$ quantifies the stability of local updates.
Notably, $\kappa \approx 1$ using a relatively small $\eta$.

For the learnability of model parameters, we assume that collectively there are sufficient data in each cluster. 
In particular, we assume $N_j \gtrsim d $, where $N_j=\sum_{i:z_i=j}n_i$ denotes the number of data points in cluster $j$.
To further characterize the \emph{quantity skewness} (i.e., the imbalance of data partition $n=(n_1,\dots,n_M)$ across clients), we adopt the $\chi^2$-divergence, which is defined as $\chi^2(P\| Q)=\int \frac{(dP-dQ)^2}{dQ}$ for
a distribution $P$ absolutely continuous with respect to a distribution $Q.$
Let $\chi^2(n)$ be the chi-squared divergence between data partition $p_n$ over the clients $p_n(i)=n_i/N$ and the uniform distribution over $[M]$. Note that when data partition is balanced (\ie, $n_i=N/M$ for all $i$), it holds that $\chi^2(n)=0.$ 

We have the following theoretical guarantee of \prettyref{phase:iterative-FL}, where $s$ is the number of local steps in FedAvg. Notably, $s$ is an algorithmic parameter for FedAvg only. To recover the results for FedProx, we only need to set $s=1$. 

\begin{theorem}
\label{thm:error-iteration}
Suppose that $k\ge 2$, $\eta \lesssim 1/s$, and $N_j\gtrsim d$.
Let $\rho = \min_j N_j/N$.
If $\nu\log(e/\nu)\lesssim \rho/\kappa$,
then with probability $1-Cke^{-d}$, for all $t\ge T+1$
and all $\theta_T$ such that $d(\theta_T,\theta^*)\le \epsilon\Delta$, where $\epsilon$ is some constant, it holds that 
\begin{align}
d(\theta_t,\theta^*)
\le \pth{1-C_1 s\eta \rho/\kappa} d(\theta_{t-1},\theta^*)
+ C_2 s\eta \sigma \nu\log\frac{e}{\nu}, \label{eq:fed_error_contraction}
\end{align}
where 
\begin{align}
\label{eq:def_nu}
\nu \triangleq \frac{1}{N}\sum_{i=1}^M n_i p_e(n_i) + C\sqrt{\frac{dk\log k}{M}(\chi^2(n)+1)},
\end{align}
and $p_e(n_i)=4ke^{-cn_i\pth{1\wedge \frac{\Delta^2}{\sigma^2}}^2} $.
%
Furthermore, if $t\ge T + 1 $, for each client $i$, with probability $1-p_e(n_i)$, it is true that 
\[
\norm{\hat{\theta}_{i,t} - \theta_{z_i}^*}
\lesssim \Delta \cdot e^{-C_1 s\eta \rho (t-T) /\kappa } + \frac{\sigma\kappa}{\rho}\nu\log\frac{e}{\nu},
\]
where $\hat{\theta}_{i,t}$ is client $i$'s estimate of its own model parameter at time $t.$
\end{theorem}
Notably, $\hat{\theta}_{i,t}$ is the $z_{i,t}$-th entry of $\theta_{i\cdot, t}$. 
\prettyref{thm:error-iteration} shows that the model estimation errors decay geometrically starting from any realization that is within a small neighborhood of 
$\theta^*$.
The parameter $\nu$ captures the additional clustering errors injected at each iteration.
It consists of two parts: the first term of~\prettyref{eq:def_nu} bounds the clustering error in expectation
which diminishes exponentially in the local data size and the signal-to-noise ratio $\Delta/\sigma$; 
the second term bounds the uniform deviation of the clustering error across all initialization and iterations.
Note that if the cluster structure were known exactly, we would get an model estimation error of $\theta_j^*$ scaling as $\sqrt{d/N_j}$. However, it turns out that this estimation error is dominated by our uniform deviation bound of the clustering error
and hence is not explicitly shown in our bound~\prettyref{eq:fed_error_contraction}. In comparison, 
the previous work~\cite{ghosh2020efficient} assumes  fresh samples at each iteration by sample-splitting and good initialization independent of everything else
provided a prior; hence their analysis fails to capture the influence of the uniform deviation of the clustering error.

 \subsubsection{Analysis of global iterations}
Without loss of generality, assume the optimal permutation in \eqref{eq:distance} is identity.
In this case, if $z_{i,t}=j$, then client $i$ will refine $\theta_{j,t-1}$.
To prove \prettyref{thm:error-iteration}, we need to analyze the global iteration of $\theta_t$.
Following a similar argument to \cite{su2021achieving} with a careful examination of cluster labels, we obtain the following lemma. The proof is deferred to \prettyref{app:global-iteration}.
\begin{lemma}
\label{lmm:global-iteration}
Let $\ph{\bx}$ be the matrix that stacks all $\ph{\bx_i}$ vertically, and similarly for $y$.
It holds that
\begin{equation}
\label{eq:global-s}
\theta_{j,t}=\theta_{j,t-1} - \eta B \Lambda_{j,t} (\ph{\bx} \theta_{j,t-1}-y),\quad j\in[k],
\end{equation}
where $B=\frac{1}{N}\ph{\bx}^\top P$,  
$P$ is a block diagonal matrix with $i\Th$ block $P_i$ of size $n_i\times n_i$ given by 
\[
P_i=
\begin{cases}
\sum_{\ell=0}^{s-1}(I-\eta \ph{\bx_i} \ph{\bx_i}^\top / n_i)^\ell & \text{for FedAvg},\\
[I+\eta \ph{\bx_i} \ph{\bx_i}^\top /n_i]^{-1} & \text{for FedProx},
\end{cases}
\]
and $\Lambda_{j,t}$ is another block diagonal matrix with $i\Th$ block being $\lambda_{ij,t} I_{n_i}$.
\end{lemma}

\prettyref{lmm:global-iteration} immediately yields the evolution of estiamtion error. 
Let $\Lambda_j$ be the matrix with $i\Th$ block being $\indc{z_i=j} I_{n_i}$ representing the true client identities.
Plugging model \eqref{eq:model}, the estimation error evolves as
\begin{equation}
\label{eq:error-iteration}
\theta_{j,t}-\theta_j^*=(I-\eta K_j)(\theta_{j,t-1}-\theta_j^*)
- \eta B\calE_{j,t}(\ph{\bx}\theta_{j,t-1}-y) + \eta B \Lambda_j \zeta,
\quad \forall j\in[k],
\end{equation}
where $K_j=B\Lambda_j \ph{\bx}$ and $\calE_{j,t}=\Lambda_{j,t}-\Lambda_j$.
The estimation error is decomposed into three terms: 1) the main contribution to the decrease of estimation error; 2) the clustering error; and 3) the noisy perturbation. 
Let $I_j=\{i:z_i=j\}$ be the clients belonging to $j\Th$ cluster, and $I_{j,t}=\{i:z_{i,t}=j\}$ be the clients with estimated label $j$. 
The indices of nonzero blocks of $\calE_{j,t}$ are $I_j \ominus I_{j,t}$ indicating the clustering errors pertaining to $j\Th$ cluster.

For the ease of presentation, we introduce a few additional notations for the collective data over a subset of clients. 
Given a subset $I\subseteq [M]$ of clients, let $\ph{\bx_{I}}$ denote the matrix that vertically stacks $\ph{\bx_i}$ for $i\in I$, and we similarly use notations $y_I$ and $\zeta_I$; let $P_I$ be the matrix with diagonal blocks $P_i$ for $i\in I$. 
Using those notations, we have $K_j=\frac{1}{N}\ph{\bx_{I_j}}^\top P_{I_j} \ph{\bx_{I_j}}$, 
which differs from the usual covariance matrix by an additional matrix $P_{I_j}$.
Therefore, the analysis of the first and third terms on the right-hand side of \eqref{eq:error-iteration} follows from standard concentration inequalities for random matrices.
In the remaining of this subsection, we focus on the second term, which is a major challenge in the analysis. 
The proof details are all deferred to \prettyref{app:proof-iterative}.

\begin{lemma}
\label{lmm:clustering-error}
There exists a universal constant $C$ such that, with probability $1-Ce^{-d}$,
\begin{equation}
\label{eq:clustering-error-norm}
\Norm{B\calE_{j,t}(\ph{\bx}\theta_{j,t-1}-y)}
\lesssim s (d(\theta_{t-1},\theta^*)+\sigma) \nu\log\frac{e}{\nu},
\qquad \forall~j\in[k].
\end{equation}
\end{lemma}

\prettyref{lmm:clustering-error} aims to upper bound the error of
\begin{equation}
\label{eq:clustering-error}
B\calE_{j,t}(\ph{\bx}\theta_{j,t-1}-y)
=\frac{1}{N}\ph{\bx_{S_{j,t}}}^\top P_{S_{j,t}}   (\ph{\bx_{S_{j,t}}}\theta_{j,t-1}-y_{S_{j,t}}),
\end{equation}
where $S_{j,t}=I_j \ominus I_{j,t}$. 
The technical difficulty arises from the involved dependency between the clustering error $S_{j,t}$ and the estimated parameter $\theta_{j,t-1}$ as estimating label $z_{i,t}$ and updating $\theta_{j,t-1}$ use a common set of local data.

\begin{proof}[Proof Sketch of \prettyref{lmm:clustering-error}.]

It follows from the definition of $z_{i,t}$ in \eqref{eq:hat-z} that
\begin{align*}
\Norm{\ph{\bx_{i}}\theta_{j,t-1}-y_{i}}
\le \Norm{\ph{\bx_{i}}\theta_{z_{i},t-1}-y_{i}},
\quad \forall i\in S_{j,t}.
\end{align*}
Then,
\begin{align}
\Norm{\ph{\bx_{S_{j,t}}}\theta_{j,t-1}-y_{S_{j,t}}}^2
& =\sum_{i\in S_{j,t}}\Norm{\ph{\bx_{i}}\theta_{j,t-1}-y_{i}}^2
\le \sum_{i\in S_{j,t}}\Norm{\ph{\bx_{i}}\theta_{z_i,t-1}-y_{i}}^2 \nonumber \\
& \le \sum_{i\in S_{j,t}}2\pth{ \Norm{\ph{\bx_{i}}(\theta_{z_i,t-1}-\theta_{z_i}^*)}^2 + \Norm{\zeta_i}^2)}\nonumber\\
& \le 2 \pth{d(\theta_{t-1},\theta^*) \cdot \Norm{\ph{\bx_{S_{j,t}}}}   + \Norm{\zeta_{S_{j,t}}} }^2.
\label{eq:error-first-step}
\end{align}

Hence, it suffices to upper bound $\Norm{\ph{\bx_{S_{j,t}}}}$ and $\Norm{{\zeta_{S_{j,t}}}}$ given a small estimation error $d(\theta_{t-1},\theta^*)$ from the last iteration.
To this end, we show a uniform upper bound of the total clustering error $\sum_{i\in S_{j,t}} n_i$ by analyzing a weighted empirical process.
Using the decision rule \eqref{eq:hat-z}, the set $S_{j,t}$ can be written as a function $S_j(\theta_{t-1})$ with
\begin{equation}
\label{eq:indc-error}
\indc{i\in S_j(\theta)}
=
\begin{cases}
\max\limits_{\ell\ne j}\indc{P_{j \ell}[\bx_i,y_i](\theta)\ge 0 }
\triangleq f^{\RN{1}}_{j,\theta}(\bx_i,y_i), 
& i\in I_j,
\\
\prod\limits_{\ell\ne j}\indc{P_{\ell j}[\bx_i,y_i](\theta)\ge 0 }
\triangleq f^{\RN{2}}_{j,\theta}(\bx_i,y_i), 
& i\not\in I_j,
\end{cases}
\end{equation}
where
\[
P_{jj'}[\bx_i,y_i](\theta)
\triangleq
\Norm{ y_i - \ph{\bx_i} \theta_j }^2-\Norm{ y_i - \ph{\bx_i} \theta_{j'} }^2.
\]
Then we derive the following uniform deviation of the incorrectly clustered data points
\[
\sup_{\theta\in\reals^{dk}}\abth{\sum_{i=1}^M n_i\indc{i\in S_j(\theta)}- \sum_{i=1}^M n_i \prob{i\in S_j(\theta)} }
\le C N\sqrt{\frac{dk\log k}{M}(\chi^2(n)+1)} \,.
\]
This is proved via upper bounds on the Vapnik–Chervonenkis (VC) dimensions of the binary function classes
\begin{equation}
\label{eq:def-calF}
\calF_j^{\RN{1}} \triangleq \{f^{\RN{1}}_{j,\theta}: \theta\in\reals^{dk} \},
\quad
\calF_j^{\RN{2}} \triangleq \{f^{\RN{2}}_{j,\theta}: \theta\in\reals^{dk} \}.
\end{equation}
Using classical results of VC dimensions, those functions are equivalently intersections of hyperplanes in ambient dimension $O(d^2)$, which yields an upper bound $O(d^2)$.
However, the hyperplanes are crucially rank-restricted as the total number of parameters in $\theta$ is $dk$. 
We prove that the VC dimensions are at most $O(dk\log k)$ using the algebraic geometry of polynomials given by the celebrated Milnor-Thom theorem (see, e.g., \cite[Theorem 6.2.1]{Matousek2013}).
Consequently, $\Norm{\ph{\bx_{S_{j,t}}}},\Norm{{\zeta_{S_{j,t}}}}$ and thus \eqref{eq:error-first-step} can be uniformly upper bounded using sub-Gaussian concentration and the union bound, concluding the proof of \prettyref{lmm:clustering-error}.

\end{proof}

\subsection{Global convergence}

Combining~\prettyref{thm:warm_start} and~\prettyref{thm:error-iteration}, we immediately deduce
the global convergence from any initialization within the $\ell_2$ ball of radius $R.$

\begin{theorem}\label{thm:main_convergence}
Suppose the conditions of~\prettyref{thm:warm_start} and~\prettyref{thm:error-iteration} hold.   
Let $\hat\theta$ be the
output of our two-phase algorithm by running \prettyref{phase:moment_descent} with $T=\Theta(1)$ iterations starting from any initialization $\theta_0$ with $\norm{\theta_0}\le R$,
followed by \prettyref{phase:iterative-FL} with $T'=\Theta(\frac{\kappa}{s\eta \rho}{\log \frac{\Delta}{\nu}})$ iterations.
Then with probability $1-N^{-9}-cke^{-d}$, it is true that 
\begin{align}
d(\hat{\theta},\theta^*)
\lesssim 
\frac{\sigma \kappa}{\rho}\nu\log\frac{e}{\nu}, \label{eq:global_error_bound}
\end{align}
Furthermore, for each client $i$, with probability $1-p_e(n_i)$, it holds that 
$
\Norm{\hat{\theta}_{i,T+T'} - \theta_{z_i}^*}
\lesssim  \frac{\sigma\kappa}{\rho}\nu\log\frac{e}{\nu}.
$
\end{theorem}
To the best of our knowledge, this is the first result that proves the global convergence of clustered federated learning
from any initialization. 
Our bound~\prettyref{eq:global_error_bound} reveals that the final estimation error is dominated by the clustering error captured by $\nu$,
and scales linearly in $\kappa$ which characterizes the stability of local updates under FedAvg or FedProx. Moreover, \prettyref{thm:main_convergence} shows that Phase 1 converges very fast with only $\Theta(1)$ iterations and hence is relatively inexpensive in both computation and communication. Instead, the number of iterations needed for Phase 2 grows logarithmically in $\Delta/\nu$ and linearly in $\kappa/(s\eta\rho)$. 
Thus, by choosing $s$ relatively large while keeping $\kappa$ close to 1, FedAvg enjoys a saving
of the total communication cost.

\appendices

\section{Coarse estimation with multiple clusters}\label{app:warm_start}
With multiple clusters and sub-Gaussian features, simple procedures such as  
power method will not provide a reasonably good coarse estimation. 
The reasons are two-fold: 1) With sub-Gaussian features, the 
 leading eigenspace of $Y$ may not align with the
 space spanned by the true model parameters $(\theta^*_1, \ldots, \theta^*_k)$; 2)
 Even if they were aligned, there would still be 
 significant ambiguity in determining the
 model parameters from their spanned subspace. As such, we develop the 
 federated moment descent algorithm (formally described in~\prettyref{phase:moment_descent}).
 In this section, we present its analysis. 

\subsection{Subspace estimation via federated orthogonal iteration}
Recall that~\prettyref{phase:moment_descent} aims to estimate the subspace that the residual estimation errors $\{\Sigma_j(\theta_j^*-\theta_{i,t})\}_{j=1}^k$ lie in  via the {\sf federated-orthogonal iteration}.
Analogous to the symmetric two-cluster case with standard Gaussian feature where $\expect{Y}$ is of rank $1$ with
the leading eigenvector  parallel to $\theta^*$, here we can show that
$\expect{Y_{i,t}}$ is of rank at most $k$ and the eigenspace corresponding to the non-zero eigenvalues is spanned by $\{\Sigma_j(\theta_j^*-\theta_{i,t})\}_{j=1}^k$. 
Specifically, we first prove that
$Y_{i,t}$ is close to $\expect{Y_{i,t}}$ in the operator norm and then further deduce that 
$\{\Sigma_j(\theta_j^*-\theta_{i,t})\}_{j=1}^k$ approximately lie in the subspace spanned by the top-$k$
left singular vectors of $Y_{i,t}$.

Let $U_{i,t} \in \reals^{d \times k}$ denote the top-$k$ left singular matrix of $Y_{i,t}$.
 To approximately compute $U_{i,t}$ in the FL systems, we adopt the following
  {\sf federated-orthogonal iteration} algorithm.
Suppose that $Y$
admits a decomposition over distributed clients, that is, $Y=\frac{1}{\sum_{i \in S} n_i} \sum_{i \in S} \sum_{j \in n_i} a_{ij} b_{ij}^\top$, where $S$ is a set of clients,
and $\{a_{ij}, b_{ij}\}_{j=1}^{n_i}$ are computable based on the local dataset $\calD_i$, 
\prettyref{alg:OI}
approximates the top-$k$ left singular matrix of $Y$. 
It can be easily verified that \prettyref{alg:OI} effectively runs the orthogonal iteration on $YY^\top.$
\begin{algorithm}[htb]
\caption{Federated orthogonal iteration}\label{alg:OI}
\textbf{Input:}  A set $S$ of clients $i$ with  $\{a_{ij}, b_{ij}\}_{i \in S, j \in[n_i]}$, 
$k \in \naturals$, and even $T \in \naturals$\\
  \textbf{Output:} $Q_T \in \reals^{d \times k}$
\begin{algorithmic}[1]
\STATE PS initialize $Q_0 \in \reals^{d\times k}$ as a random orthogonal matrix $Q_0^\top Q_0=\identity.$
\FOR{$t=0,1, \ldots, T-1$}
\STATE PS broadcasts $Q_t$ to all clients in $S.$
\IF{$t$ is even}
\STATE Each client $i \in S$ computes an update $Q_{i,t}= \frac{1}{n_i} \sum_{j=1}^{n_i} b_{ij} a_{ij}^\top Q_t$ and transmits it back to the PS.
\STATE PS updates $Q_{t+1}= \sum_{i \in S}  w_i Q_{i,t},$ where $w_i=n_i/\sum_{i\in S} n_i$.
\ELSE
\STATE Each client computes an update $Q_{i,t}= \frac{1}{n_i} \sum_{j=1}^{n_i} a_{ij} b_{ij}^\top Q_t$ and transmits it back to the PS.
\STATE PS applies the QR decomposition to obtain $Q_{t+1}$: 
$$
\sum_{i \in S} w_i Q_{i,t} =Q_{t+1} R_{t+1}.
$$
\ENDIF
\ENDFOR
\end{algorithmic}
\end{algorithm}

Recall that~\prettyref{phase:moment_descent} is called for each anchor client $i\in H$ and each global iteration $t$, and that $\hat{U}_{i,t}$ is the output of {\sf federated-orthogonal iteration} in  Step 9 of~\prettyref{phase:moment_descent},
which  approximates $U_{i,t}$.
Based on the above discussion, we can show that 
the residual estimation errors $\{ \Sigma_j(\theta_j^*- \theta_{i,t}) \}_{j=1}^k$ approximately lie in the subspace spanned by the $k$ columns of $\hat{U}_{i,t}.$
\begin{proposition}[Subspace estimation]\label{prop:sub_est}
If $T_1 \ge C k \log (Nd)$ for some sufficiently large constant $C$, then with probability at least $1-N^{-10},$
$$
\norm{ \left( \hat{U}_{i,t} \hat{U}_{i,t}^\top - I \right)  \Sigma_j \left( \theta_j^*- \theta_{i,t} \right)}^2 \le O\left( \left( \delta_{i,t}^2+\sigma^2 \right) \xi_1/p_j \right), \quad \forall j \in [k],
$$
where  $\delta_{i,t}=\max_{j \in [k]} \Norm{\theta_j^*-\theta_{i,t}}$ and $\xi_1=
\sqrt{ \frac{d}{m} \log N } + \frac{d}{m} \log^3 N.$
\end{proposition}
We postpone the detailed proof to \prettyref{app:proof_warm_start_1}. 
One key challenge in the analysis is that  the eigengap of $\expect{Y_{i,t}}$ could be small, especially when $\theta_{i,t}$ approaches $\theta^*_{z_i}$;
and hence the standard Davis-Kahan theorem cannot be applied. This issue is
further exaggerated by the fact that the convergence rate of the orthogonal iteration also crucially depends on the eigengaps. To resolve this issue, one key innovation of our analysis is to develop a gap-free bound to show that
the projection errors $ \hat{U}_{i,t}^\top \Sigma_j(\theta_j^*- \theta_{i,t})$ are small for every $j \in [k]$ (cf.~\prettyref{lmm:DK_Power}).

\subsection{Moment descent on anchor clients}
Recall that in Step 11 of~\prettyref{phase:moment_descent}, each anchor client $i\in H$ runs the {\sf power-iteration} to output $\hat{\beta}_{i,t}$ and $\hat{\sigma}^2_{i,t}$ as approximations of the leading left singular vector
 and singular value of $ A_{i,t} $, respectively.
Then anchor client $i$ updates a new estimate $\theta_{i,t+1}$ by moving along the direction of the estimated residual error $r_{i,t}$ with an appropriately chosen step size $\eta_{i,t}$.
The following result shows that $\hat{\sigma}^2_{i,t}$ closely approximates the squared residual error
$\Norm{\Sigma_j( \theta_{z_i}^*- \theta_{i,t} ) }^2$. Moreover, `the residual  error 
decreases geometrically until reaching a plateau.
\begin{proposition}\label{prop:MD}
Fix an anchor client $i$ and suppose $T_2 \gtrsim \log (Nd)$. There exists a constant $C>0$ and an event $\calE_{i,t}$ with
$\prob{\calE_{i,t}} \ge 1-O(N^{-10})$
such that on event $\calE_{i,t}$
\begin{align}
\left| \norm{\Sigma_j( \theta_{z_i}^*- \theta_{i,t} ) }^2 - \hat{\sigma}_{i,t}^2 \right|
\le C(\delta_{i,t}^2 +\sigma^2) (\xi_1/p_{z_i} + \xi_2) , \label{eq:sigma_approx}
\end{align}
where $\xi_2= \sqrt{ \frac{k}{\ell} \log N } + \frac{k}{\ell} \log^3 N .$
Furthermore,  if
\begin{align}
\norm{ \Sigma_{z_i} \left( \theta_{z_i}^*- \theta_{i,t} \right)}^2 \ge C
(\xi_1/p_{z_i}+  \xi_2), \label{eq:stop_cond_warm_up}
\end{align}
then 
\begin{align}
\norm{ \Sigma_{z_i} \left( \theta_{z_i}^* - \theta_{i,t+1} \right)}^2
\le \left( 1 - \frac{\alpha^2 }{ 8 \beta^2}  \right) \norm{ \Sigma_{z_i} \left( \theta_{z_i}^* - \theta_{i,t} \right)}^2, \label{eq:moment_decay}
\end{align}
We postpone the proof to~\prettyref{app:proof_warm_start_2}. The key ingredient in the proof of geometric decay~\prettyref{eq:moment_decay} is to show that
the descent direction $r_{i,t}$ is approximately parallel to the residual error 
$\Sigma_{z_i} ( \theta_{z_i}^* - \theta_{i,t} )$
under condition~\prettyref{eq:stop_cond_warm_up}

\end{proposition}

\subsection{Proof of~\prettyref{thm:warm_start}}
Now, we are ready to prove our main theorem on the performance guarantee of~\prettyref{phase:moment_descent}.
Let $\calE_{i,t}$ denote the event under which the statement of \prettyref{prop:MD} holds. Let $\calE=\cup_{i \in H} \cup_{t=1}^T \calE_{i,t}$. 
By the union bound, 
$\prob{\calE}\ge 1- O(n_H T/N^{10}) \ge 1- O(N^{-9})$. In the following, we assume event $\calE$ holds.

We first prove \prettyref{eq:warm_start_error}. 
Fix any anchor client $i$ and omit the subscript $i$ for simplicity. We further assume it belongs to cluster $j$, \ie, $z_i=j.$
Define
$$
t^*=\min\{ \inf\{ t \ge 0: \hat{\sigma}_t\le \epsilon \Delta\}, T\}.
$$
By definition,
\begin{align}
\hat{\sigma}_t
> \epsilon \Delta, \quad \forall 0 \le t \le t^* -1. \label{eq:sigma_bound_1}
\end{align}
Moreover, by the update rule of our algorithm,
$\theta_T=\theta_{t^*}.$ 
Thus, it suffices to bound $\Norm{\theta_{t^*}-\theta^*_j}$. 

We claim that  for all $ 0 \le t \le t^*$,
\begin{align}
\norm{ \Sigma_{j} \left( \theta_{j}^* - \theta_{t} \right)}^2 & \le 
 \left( 1 - \frac{\alpha^2 }{ 8 \beta^2}  \right)^t \norm{ \Sigma_{j}  \left( \theta_{j}^* - \theta_{0} \right) }^2 \label{eq:induction_1} \\
\Norm{\theta_t} & \le (1+2\beta/\alpha) R. \label{eq:induction_2} 
\end{align}
If $t^*=T$, then by~\prettyref{eq:induction_1}
we immediately get that 
$$
\norm{ \Sigma_{j} \left( \theta_{j}^* - \theta_{t^*} \right)}  \le 
 \left( 1 - \frac{\alpha^2 }{ 8 \beta^2}  \right)^{T/2} \norm{ \Sigma_{j} \theta_{j}^*}
 \le \exp \left( - T \alpha^2/(16\beta^2)\right) \beta R  \le \epsilon \Delta,
 $$
where the last inequality holds by choosing $T= \frac{16\beta^2}{\alpha^2}
\log \frac{\beta R}{\epsilon \Delta}$. 
 
If $t^\ast <T,$ then by~\prettyref{eq:induction_2}
and $\Norm{\theta^*_j}\le R$ for all $j \in [k],$
it follows that $\delta_t \le 2(1+\beta/\alpha)R$. 
Therefore, by~\prettyref{eq:sigma_approx}
\begin{align}
\left| \norm{\Sigma_j( \theta_j^*- \theta_t ) }^2 - \hat{\sigma}_{t}^2 \right|
\le C(\delta_t^2 +\sigma^2) (\xi_1/p_j + \xi_2)  
\le \epsilon^2 \Delta^2 /2, \label{eq:induction_3}
\end{align}
where the last inequality holds by choosing $m=\tilde{\Omega}(d/p_{\min}^2)$ and $\ell=\tilde{\Omega}(k)$
and invoking the standing assumption that $R=O(\Delta)$ and $\sigma=O(\Delta).$
It immediately follows that 
$$
 \norm{\Sigma_j( \theta_j^*- \theta_{t^\ast} ) }^2 \le \hat{\sigma}_{t^\ast}^2 + \epsilon^2 \Delta^2 /2 
 \le \frac{3}{2}  \epsilon^2 \Delta^2,
$$
where the last inequality holds by the stopping rule of our algorithm so that  $\hat{\sigma}_{t^\ast} \le \epsilon \Delta.$ 

In both cases, we get that 
$$
\norm{\theta_j^*- \theta_{t^\ast}} \le \frac{1}{\alpha} \norm{\Sigma_j( \theta_j^*- \theta_{t^\ast} ) }
\le  \frac{ 2\epsilon \Delta}{ \alpha} = \epsilon' \Delta 
$$
for $\epsilon'=2\epsilon/\alpha.$ This proves~\prettyref{eq:warm_start_error}.

Now, it remains to prove the claim~\prettyref{eq:induction_1}--\prettyref{eq:induction_2} by induction. 
The base case $t=0$ trivially holds as $\norm{\theta_0}\le R.$ Now, suppose the induction hypothesis
holds for an arbitrary $t$ where $0 \le t \le t^*-1$, we prove it also holds for $t+1$.
In view of~\prettyref{eq:induction_3},
$$
\norm{\Sigma_j( \theta_j^*- \theta_t ) }^2  \ge \hat{\sigma}_t^2 - \epsilon^2 \Delta^2/2  > \epsilon^2 \Delta^2/2
 \ge C(\delta^2_t +\sigma^2) (\xi_1/p_j + \xi_2),
$$
where the second inequality holds due to~\prettyref{eq:sigma_bound_1}.
Therefore, the condition~\prettyref{eq:stop_cond_warm_up} is satisfied. 
Hence, by applying~\prettyref{prop:MD}, we get that
$$
\norm{ \Sigma_{j} \left( \theta_{j}^* - \theta_{t+1} \right)}^2
\le \left( 1 - \frac{\alpha^2 }{ 8 \beta^2}  \right) \norm{ \Sigma_{j} \left( \theta_{j}^* - \theta_{t} \right)}^2
\le \left( 1 - \frac{\alpha^2 }{ 8 \beta^2}  \right)^{t+1} \norm{ \Sigma_{j} \left( \theta_{j}^* -\theta_0 \right)}^2.
$$
where the second inequality holds by the induction hypothesis~\prettyref{eq:induction_1}.
Hence
 $$
 \alpha^2 \norm{\theta_{j}^* - \theta_{t+1}}^2 \le \norm{ \Sigma_{j} \left( \theta_{j}^* - \theta_{t+1} \right)}^2 \le 
 \norm{ \Sigma_{j} \left( \theta_{j}^* - \theta_{0} \right)}^2 \le 4 \beta^2 R^2.
 $$
It follows that 
$$
\Norm{\theta_{j}^* - \theta_{t+1}} \le \frac{2\beta}{\alpha} R
$$
and hence $\Norm{\theta_{t+1}} \le (1+2\beta/\alpha) R.$ 
This completes the induction proof. 

Finally, we prove~\prettyref{eq:warm_start_cluster}.
Note that by standard coupon collector's problem, we deduce that if $n_H \ge \log (k/\delta)/p_{\min}, $ then with probability at least $1-\delta$, 
$H \cap \{ i: z_i=j \} \neq \emptyset$ for all $ j \in [k]$. To see this, note that
\begin{align*}
\prob{ H \cap \{ i: z_i=j \} \neq \emptyset, \, \forall j \in [k] }
& \ge 1- \sum_{ j \in [k] }\prob{ H \cap \{ i: z_i=j \} = \emptyset} \\
& \ge 1- k \left(1-p_{\min}\right)^{n_H} \\
& \ge 1- k \exp \left( - p_{\min} n_H \right) \ge 1-\delta.
\end{align*}
Therefore, as long as $\epsilon' < 1/4$, we have for two anchor clients $i, i' \in H$
\begin{align*}
\norm{\theta_{i,T} - \theta_{i', T} } \le \norm{\theta_{i,T}-\theta^*_{z_i} } +  \norm{\theta_{i',T}-\theta^*_{z_{i'}}}
\le 2\epsilon' \Delta, \quad \text{ if } z_i = z_{i'} \\
\norm{\theta_{i,T} - \theta_{i', T} } \ge \Delta - \norm{\theta_{i,T}-\theta^*_{z_i} } -  \norm{\theta_{i',T}-\theta^*_{z_{i'}}} \ge
(1-2\epsilon') \Delta, \quad \text{ if } z_i \neq z_{i'}.
\end{align*}
Thus, by assigning anchor clients $i,i' \in H$ in the same cluster when
$\norm{\theta_{i,T}-\theta_{i',T}} \le \Delta /2$
we can exactly recover the $k$ clusters of the clients users. In particular, 
let $\hat{z}_i$ denote the estimated cluster label of anchor client $i \in H.$
Then there exists a permutation $\pi: [k] \to [k]$ such that $\pi(\hat{z}_i) = z_i$ for all $i \in H.$
Let $\hat{\theta}_j$ denote the center of the recovered cluster $j$, that is 
$$
\hat{\theta}_j = \sum_{i \in H} \theta_{i,T}\indc{\hat{z}_i=j} / \sum_{i \in H} \indc{\hat{z}_i=j}. 
$$
Then we have $\Norm{\hat{\theta}_{\pi(j)} - \theta^*_j}\le \epsilon' \Delta$ for all $j \in [k].$
This finishes the proof of~\prettyref{eq:warm_start_cluster}.



\subsection{Proof of~\prettyref{prop:sub_est}}\label{app:proof_warm_start_1}
In the following analysis, we fix an anchor client $i \in H$ and omit the subscript $i$ for ease of presentation. Crucially,  since $\calS_{t}$ and $\calD_{t}$ are freshly drawn, all the global data and local data used in iteration $t+1$ are independent from $\theta_{t}$. 
Hence, we condition on $\theta_t$ in the following analysis.  
Note that
$$
\expect{Y_t}= \frac{1}{m} \sum_{i' \calS_t} 
\expects{\Sigma_{z_{i'}} \left( \theta_{z_{i'}}^*- \theta_t \right) \left( \theta_{z_{i'}}^*- \theta_t \right)\Sigma_{z_{i'}} }{z_{i'}}= 
\sum_{j=1}^k p_j \Sigma_j \left( \theta_j^* - \theta_t \right) \left( \theta_j^* - \theta_t \right)^\top \Sigma_j,$$
where $p_j$ is the probability that a client belongs to the $j$-th cluster. 


Let $U_t\in \reals^{d\times k}$ denote the left singular matrix of $Y_t$. 
We aim to show that the collection of $\Sigma_j ( \theta_j^* - \theta_t)$ for $ j \in [k]$ approximately lie in the space spanned by the $k$ columns of $U_t$. As such, we first show that $Y_t$ is close to $\expect{Y_t}$ in operator norm. 
\begin{lemma}\label{lmm:concentration_Y_t}
With probability at least $1-3N^{-10}$, 
$$
\norm{Y_t-\expect{Y_t}} \le O\left( \left(\delta_t^2+\sigma^2\right) \xi_1 \right),
$$
where $\delta_t=\max_{j \in [k]} \Norm{\theta_j^*-\theta_t}$ and $\xi_1=
\sqrt{ \frac{d}{m} \log N } + \frac{d}{m} \log^3 N.$
\end{lemma}
\begin{proof}
Let $\varepsilon_i= (y_{i1} - \iprod{\ph{x_{i1}}}{\theta_t}) \ph{x_{i1}}$
and $\tilde{\varepsilon}_i= (y_{i2} - \iprod{\ph{x_{i2}}}{\theta_t}) \ph{x_{i2}}$.
Note that 
\begin{align*}
    Y_t - \expect{Y_t} 
    &= \frac{1}{m} \sum_{i=1}^m 
    \varepsilon_i \tilde{\varepsilon}_i^\top  - \expect{\varepsilon_i \tilde{\varepsilon}_i^\top}.
\end{align*}
Let $a_i=\varepsilon_i/\sqrt{\delta_t^2+\sigma^2}$ and $b_i=\tilde{\varepsilon}_i /\sqrt{\delta_t^2+\sigma^2}$. 
We will apply a truncated version of the Matrix Bernstein's inequality given in~\prettyref{lmm:matrix_bernstein_2}. As such, we first check
the conditions in~\prettyref{lmm:matrix_bernstein_2} are all satisfied.  
Note that 
\begin{align*}
\expect{\norm{\varepsilon_i}^2} & = \expect{ \norm{ \left( \Iprod{\ph{x_{i1}}}{\theta^*_{z_i}-\theta_t}+\zeta_i \right) \ph{x_{i1}} }^2} \\
&= \expect{ \norm{ \Iprod{ \ph{x_{i1}} }{\theta^*_{z_i}-\theta_t} \ph{x_{i1}} }^2} + \expect{ \norm{\zeta_{i1} \ph{x_{i1}}}^2}.
\end{align*}
By the sub-Gaussianity of $\ph{x_{i1}}$, we  have
$$
\expect{ \norm{\zeta_i \ph{x_{i1}}}^2} \le  \sigma^2 \expect{\norm{\ph{x_{i1}}}^2} = O(\sigma^2 d)
$$
and further by Cauchy-Schwarz inequality,
$$
 \expect{ \norm{ \Iprod{\ph{x_{i1}} }{\theta^*_{z_i}-\theta_t} \ph{x_{i1}} }^2}
 \le \sqrt{\expect{ \Iprod{\ph{x_{i1}} }{\theta^*_{z_i}-\theta_t}^4 }}
 \sqrt{ \expect{\norm{\ph{x_{i1}}}^4}} = O\left(\delta_t^2 d \right).
$$
Combining the last three displayed equation gives that 
$
\expect{\norm{a_i}^2} \le O\left(  d \right).
$
The same upper bound also holds for $\expect{\norm{b_i}^2}.$

Moreover,
$
\norm{\expect{a_i a_i^\top}} =\sup_{u \in \calS^{d-1}} \expect{\Iprod{a_i}{u}^2}.
$
Note that for any $u \in \calS^{d-1}$,
$$
\expect{\Iprod{a_i}{u}^2}=  \frac{1}{\delta_t^2+\sigma^2}  \expect{ r_i^2 \Iprod{\ph{x_{i1}}}{u}^2}
\le  \frac{1}{\delta_t^2+\sigma^2}  \sqrt{\expect{r_i^4} }\sqrt{ \Iprod{\ph{x_{i1}}}{u}^4 }
\le  O\left(1 \right) ,
$$
where $r_i= 
y_{i1} - \iprod{\ph{x_{i1}}}{\theta_t}.$
Combining the last two displayed equations gives that 
$
\norm{\expect{a_i a_i^\top}} = O\left( 1 \right).
$
The same upper bound also holds for $\norm{\expect{b_i b_i^\top}}.$
Finally, by the sub-Gaussian property of $\ph{x_{i1}}$, we have
$$
\prob{ \norm{\ph{x_{i1}}} \ge s_1 } \le \exp \left( O(d) - \Omega(s_1^2) \right) 
$$
and 
$$
\prob{ \frac{|r_i|}{\sqrt{\delta_t^2+\sigma^2}} \ge s_2 } \le \exp \left( -\Omega \left( s_2^2 \right) \right) .
$$
Choosing $s_1= C \sqrt{s} d^{1/4}$ and $s_2= \sqrt{s} / (C d^{1/4})$ for a sufficiently large constant $C$, we get that for all $s \ge \sqrt{d} $,
\begin{align*}
\prob{ \norm{a_i} \ge s} & \le \prob{ \norm{\ph{x_{i1}}} \ge s_1 } + \prob{  \frac{|r_i|}{\sqrt{\delta_t^2+\sigma^2}} \ge s_2} \\
& \le \exp \left( O(d) - \Omega( C s \sqrt{d} ) \right) + 
\exp \left( - \Omega \left( \frac{s}{ \sqrt{d}   } \right) \right)\\
& \le \exp \left( - \Omega \left( \frac{s}{ \sqrt{d}  } \right) \right).
\end{align*}
The same bound holds for $\prob{ \norm{b_i} \ge s}$.
Applying the truncated version of the Matrix Bernstein's inequality given in~\prettyref{lmm:matrix_bernstein_2} yields the desired result. 
\end{proof}

The following result generalizes \prettyref{lmm:power_iteration} and shows the geometric convergence of orthogonal iteration. Let $Y=U\Lambda U^\top$ denote the eigenvalue decomposition of $Y$ with $|\lambda_1| \ge |\lambda_2| \ge \cdots |\lambda_d|$ and the corresponding eigenvectors $u_i$'s.  Define $U_1=[u_1, \ldots, u_k]$
and $U_2=[u_{k+1}, \ldots, u_d]$. 
\begin{lemma}\cite[Theorem 8.2.2]{Golub-VanLoan2013}
\label{lmm:orthogonal_iteration}
Assume $|\lambda_{k}|> |\lambda_{k+1}| $ and
$\cos(\gamma)=\sigma_{\min}(U_1^\top Q_0)$ for $\gamma \in [0, \pi/2]$.
Then 
$$
\norm{Q_t Q_t^\top -U_1 U_1^\top} \le 
\tanh(\theta) \left| \frac{\lambda_{k+1}}{\lambda_k} \right|^t, \quad \forall t.
$$
\end{lemma}
Finally, we need a gap-free bound that controls the projection errors.
\begin{lemma}[Gap-free bound on projection errors]
\label{lmm:DK_Power}
Suppose $M \in \reals^{d \times d}$ satisfies that 
$$
\norm{M-\sum_{i=1}^k x_i x_i^\top} \le \epsilon,
$$
where $x_i \in \reals^d$ for $1 \le i \le k$. Let $Q_t$ be the output of the orthogonal iteration running over $MM^\top.$
Assume that $\norm{x_i} \le H$ for all $ 1 \le i \le k$. 
There exists a universal constant $C>0$ such that for any $\epsilon>0$ and $t \ge C k \log \frac{dN H}{\epsilon}$, we have with probability at least $1-O(N^{-10}),$
$$
\norm{Q_t Q_t^\top x_i-x_i} \le 3 \sqrt{\epsilon}, \quad \forall 1 \le i \le k.
$$
\end{lemma}
\begin{proof}
Let $\sigma_1 \ge \sigma_2 \ge \ldots \ge \sigma_d \ge 0$ denote the singular values of $M$. 
Then by assumption on $M$ and Weyl's inequality, $\sigma_{k+1} \le \epsilon.$ 
We divide the analysis into two cases depending on the value of $\sigma_1. $ Let $\delta>0$ be some parameter to be tuned later. 

{\bf Case 1}: $\sigma_1 \le \left(1+\delta\right)^k \epsilon$. In this case, by Weyl's inequality,
$$
\norm{x_i}^2 \le \norm{\sum_{i=1}^k \ph{x_i} x_i^\top} \le \norm{M} + \norm{M-\sum_{i=1}^k x_i x_i^\top}
\le \epsilon \left( 1 + \left(1+\delta\right)^k \right).
$$
Thus, 
$$
\norm{Q_t Q_t^\top x_i-x_i}  \le \norm{x_i} \le \sqrt{\epsilon \left( 1 + \left(1+\delta\right)^k \right)}
$$

{\bf Case 2}: $\sigma_1 > \left(1+\delta\right)^k \epsilon$. Then by the pigeonhole principle there must exist $1 \le p \le k$ such that $\sigma_p/\sigma_{p+1} > 1+\delta.$  Choose 
$$
\ell= \max \left\{p: \sigma_p/\sigma_{p+1} > 1+\delta\right\}.
$$
It follows that $\sigma_{\ell+1} \le (1+\delta)^{k-\ell} \epsilon \le (1+\delta)^k \epsilon.$
Let $U_\ell=[u_1, \ldots, u_\ell]$, where $u_i$'s are the left singular vectors of $M$ corresponding to $\sigma_i$.
Given the subspace $\spn \{u_1,\dots,u_\ell\}$, denote the unique orthogonal decomposition of $x_i$ by $x_i=\Pi_W(x_i)+e$, where $\Pi_W(x_i)=U_\ell U_\ell^\top x_i$ and $e^\top u_j=0$ for all $j \in [\ell]$.
Let $u= e/\Norm{e} \in S^{d-1}$.
Then,
$$
\norm{U_\ell U_\ell^\top x_i -x_i }^2 = u^\top x_i x_i^\top u 
\le u^\top \left( \sum_{i=1}^k x_i x_i^\top\right) u   
=u^\top \left( \sum_{i=1}^k  x_i x_i^\top -M\right) u+ u^\top M u.
$$
Note that 
$$
u^\top  \left( \sum_{i=1}^k x_i x_i^\top -M\right) u
\le \norm{\sum_{i=1}^k x_i x_i^\top -M} \le \epsilon.
$$
Moreover, 
\begin{align*}
u^\top M u = \sum_j \sigma_j u^\top u_j v_j^\top u 
= \sum_{j\ge \ell+1} \sigma_j u^\top u_j v_j^\top u 
&\le \sigma_{\ell+1}  \sum_{j\ge \ell +1} \abth{u^\top u_j} \abth{v_j^\top u} \\
& \le \sigma_{\ell+1} \sqrt{ \sum_{j \ge \ell+1} \abth{u^\top u_j}^2
\sum_{j \ge \ell+1} \abth{v_j^\top u}^2} \\
& \le \sigma_{\ell+1} \le  (1+\delta)^k \epsilon.
\end{align*}
Combining the last three displayed equations gives that 
$$
\norm{U_\ell U_\ell^\top x_i -x_i }^2 \le  \epsilon \left( 1 + \left(1+\delta\right)^k \right).
$$
Let $\hat{Q}_t$ be the submatrix of $Q_t$ formed by the first $\ell$ columns. 
Since $\sigma_{\ell}> \sigma_{\ell+1}$, 
the space spanned by $\hat{Q}_t$ is the same space spanned by $Q_t$ if the orthogonal iteration were run with $k$ replaced by $\ell.$ Thus, applying \prettyref{lmm:orthogonal_iteration} with $k$ replaced by $\ell$ gives that 
$$
\norm{\hat{Q}_t \hat{Q}_t^\top - U_\ell U_\ell^\top } \le \tan(\gamma) (1+\delta)^{-t} ,
$$
where $\cos(\gamma)=\sigma_{\min}(U_\ell^\top \hat{Q}_0)$ and $\hat{Q}_0$ is the submatrix of $Q_0$ formed by its first $\ell$ columns.
Applying~\prettyref{lmm:subspace_angle}, we get that $\tanh(\gamma) =O(N^{10}d)$ with probability at least $1-O(N^{-10}).$
Therefore, when $ t \ge (C/\delta) \log \frac{Nd H}{\epsilon}$, we have
$$
\norm{\hat{Q}_t \hat{Q}_t ^\top - U_\ell U_\ell^\top } \le \epsilon/H.
$$
Therefore, by triangle's inequality, 
\begin{align*}
\norm{Q_t Q_t^\top x_i -x_i }
& \le \norm{\hat{Q}_t \hat{Q}_t^\top x_i -x_i } \\
& \le \norm{U_\ell U_\ell^\top x_i -x_i } + \norm{ \left( \hat{Q}_t \hat{Q}_t ^\top - U_\ell U_\ell^\top \right) x_i } \\
& \le \sqrt{\epsilon \left( 1 + \left(1+\delta\right)^k \right)} + \epsilon.
\end{align*}
Finally, choosing $\delta= 1/k$ and noting that $(1+\delta)^t \le e$, we get the desired conclusions.
\end{proof}

Applying~\prettyref{lmm:concentration_Y_t} and \prettyref{lmm:DK_Power} and invoking the assumption that $T_1 \gtrsim k \log (Nd)$, we have
with probability at least $1-O(1/N),$
$$
\norm{ \left( \hat{U}_t \hat{U}_t^\top - I \right) \sqrt{p_j} \Sigma_j \left( \theta_j^*- \theta_t \right)}^2 \le O\left( \left( \delta_t^2+\sigma^2 \right) \xi_1 \right),
$$
or equivalently, 
\begin{align}
\norm{ \hat{U}_t^\top \Sigma_j \left( \theta_j^*- \theta_t \right)}^2
\ge \norm{ \Sigma_j \left( \theta_j^*- \theta_t \right)}^2 -  O\left( \left( \delta_t^2+\sigma^2 \right) \xi_1/p_j \right). \label{eq:power_iteration_1}
\end{align}

\subsection{Proof of~\prettyref{prop:MD}}\label{app:proof_warm_start_2}

Similar to the proof of~\prettyref{prop:sub_est}, for ease of exposition, 
we fix an anchor client $i$ and omit the subscript $i$ for simplicity. We further assume client $i$ belongs to cluster $j$, \ie, $z_i=j$. 
Note that crucially, the global data points
on  clients $\calS_t$ are independent from the local data points on $\calD_{t}$.
Thus, in the following analysis, we further condition on $\hat{U}_{t}$. 
Then
$$
\expect{A_t}=\hat{U}_t^\top\Sigma_j \left( \theta_j^*- \theta_t \right) \left( \theta_j^*- \theta_t \right)^\top \Sigma_j \hat{U}_t.
$$
\begin{lemma}\label{lmm:concentration_A_t}
With probability at least $1-3N^{-10}$,
$$
\norm{A_t-\expect{A_t}} \le O\left( \left(\norm{\theta_j^*-\theta_t}^2 +\sigma^2 \right) \xi_2\right),
$$
where $\xi_2= \sqrt{ \frac{k}{\ell} \log N } + \frac{k}{\ell} \log^3 N  .$
\end{lemma}
\begin{proof}
Note that 
$$
 A_t-\expect{A_t} = \frac{1}{\ell} \sum_{j \in \calD_t} 
\hat{U}_t^\top \left( \varepsilon_j \tilde{\varepsilon}_j^\top -\expect{ \varepsilon_j \tilde{\varepsilon}_j^\top} \right)
\hat{U}_t,
$$
where $\varepsilon_j=(y_j-\ph{x_j})\ph{x_j}$ and
$\tilde{\varepsilon}_j=(\tilde{y}_j-\ph{\tilde{x}_j})\ph{\tilde{x}_j}$.
Let $a_j= \hat{U}_t ^\top \varepsilon_j / \sqrt{ \Norm{\theta_j^*-\theta_j}^2 +\sigma^2 } $
and $b_j=\hat{U}_t^\top \tilde{\varepsilon}_j /\sqrt{ \Norm{\theta_j^*-\theta_j}^2 +\sigma^2 } $.
The rest of the proof follows analogously as that of~\prettyref{lmm:concentration_Y_t}.
\end{proof}

Applying~\prettyref{lmm:concentration_A_t} and~\prettyref{lmm:DK_Power}, when $T_2 \gtrsim \log (Nd)$, 
we have with probability at least $1-O(N^{-10})$
$$
\left|\hat{\beta}_t^\top \hat{U}_t^\top \Sigma_j \left( \theta_j^*- \theta_t \right) \right|^2
\ge  \norm{ \hat{U}_t^\top \Sigma_j \left( \theta_j^*- \theta_t \right)}^2 -  O\left( \left(\norm{\theta_j^*-\theta_t}^2 +\sigma^2 \right)\xi_2 \right) . 
$$
Applying~\prettyref{prop:sub_est}, we have with probability at least $1-O(N^{-10})$
$$
\norm{ \hat{U}_t^\top \Sigma_j \left( \theta_j^*- \theta_t \right)}^2
\ge \norm{ \Sigma_j \left( \theta_j^*- \theta_t \right)}^2 - O\left( \left( \delta_t^2+\sigma^2 \right) \xi_1/p_j \right).
$$
Let $\calE_{t}$ denote the event such that the above two displayed equations hold simultaneously.
Then $\prob{\calE_{t}} \ge 1- O(N^{-10})$. In the following, we assume event $\calE_t$ holds. 

Combining the last two displayed equations yields that 
$$
\norm{ \Sigma_j \left( \theta_j^*- \theta_t \right)}^2  -O\left( \left( \delta_t^2+\sigma^2 \right) \left( \xi_1/p_j + \xi_2\right)\right) \le \left|\hat{\beta}_t^\top \hat{U}_t^\top \Sigma_j \left( \theta_j^*- \theta_t \right) \right|^2 \le \norm{ \Sigma_j \left( \theta_j^*- \theta_t \right)}^2. 
$$
Moreover, since
$$
\hat{\sigma}^2_t = \hat{\beta}_t^\top \hat{U}_t^\top A_t \hat{U}_t \hat{\beta}_t 
=\hat{\beta}_t^\top \hat{U}_t^\top \expect{A_t} \hat{U}_t \hat{\beta}_t
+ \hat{\beta}_t^\top \hat{U}_t^\top \left( A_t - \expect{A_t} \right) \hat{U}_t \hat{\beta}_t,
$$
it follows that 
$$
\left| \hat{\sigma}^2_t -\left| \hat{\beta}_t^\top \hat{U}_t^\top \Sigma_j \left( \theta_j^*- \theta_t \right) \right|^2 \right| \le  O\left( \left(\delta_t^2 +\sigma^2 \right)\xi_2 \right).  
$$
Combining the last two displayed equations yields that 
$$
\left| \hat{\sigma}^2_t  -\norm{ \Sigma_j \left( \theta_j^*- \theta_t \right)}^2 \right| 
\le O\left( \left( \delta_t^2+\sigma^2 \right) \left( \xi_1/p_j + \xi_2\right)\right). 
$$
This proves~\prettyref{eq:sigma_approx}.

Under condition~\prettyref{eq:stop_cond_warm_up},
we have
\begin{align}
\left| \hat{\beta}_t^\top \hat{U}_t^\top \Sigma_j \left( \theta_j^*- \theta_t \right) \right|^2
\ge \left( 1- \frac{\alpha^2}{64\beta^2} \right) \norm{ \Sigma_j \left( \theta_j^*- \theta_t \right)}^2 \label{eq:moment_direction}
\end{align}
and 
\begin{align}
\left( 1- \frac{\alpha^2}{32\beta^2} \right) \norm{ \Sigma_j \left( \theta_j^*- \theta_t \right)}^2 \le \hat{\sigma}^2_t \le \left( 1+ \frac{\alpha^2}{32\beta^2}\right) \norm{ \Sigma_j \left( \theta_j^*- \theta_t \right)}^2 \label{eq:sigma_t_bound}
\end{align}


Now we show that $\theta_t$ converges to $\theta_j^*$. 
Note that 
$$
\left( \theta_j^* - \theta_{t+1} \right)^\top \Sigma^2_j \left( \theta_j^* - \theta_{t+1} \right)
=\left( \theta_j^* - \theta_t \right)^\top \Sigma^2_j \left( \theta_j^* - \theta_t \right)
-2 \eta_t \left( \theta_j^* - \theta_t \right)^\top \Sigma^2_j r_t + \eta_t^2 r_t^\top \Sigma^2_j r_t.
$$
In view of~\prettyref{eq:moment_direction}, 
and recalling $r_t = \hat{U}_t \hat{\beta}_t$, we have $\norm{r_t}=1$ and under condition~\prettyref{eq:stop_cond_warm_up}
$$
\iprod{r_t}{ \Sigma_j \left( \theta_j^* - \theta_t \right) }^2 \ge \left( 1- \frac{\alpha^2}{64\beta^2} \right) \norm{ \Sigma_j \left( \theta_j^* - \theta_t \right)}^2.
$$
We decompose
$$
\Sigma_j \left( \theta_j^* - \theta_t \right) = a_t r_t + b_t r_t^\perp,
$$
for some unit vector $r_t^\perp$ that is perpendicular to $r_t$. Since $a_t^2+b_t^2=\Norm{\Sigma_j ( \theta_j^* - \theta_t ) }^2$, we have 
$|b_t| \le \frac{\alpha}{8\beta} \Norm{\Sigma_j ( \theta_j^* - \theta_t ) }$. Hence,
\begin{align*}
\left( \theta_j^* - \theta_t \right)^\top \Sigma^2_j r_t 
& = \left(a_t r_t + b_t r_t^\perp \right)^\top \Sigma_j r_t \\
& \ge  a_t \alpha - |b_t| \beta \\
& \ge 
\sqrt{1- \frac{\alpha^2}{64\beta^2}} \alpha
\norm{\Sigma_j ( \theta_j^* - \theta_t )} - \frac{\alpha}{8} \norm{\Sigma_j ( \theta_j^* - \theta_t ) } \\
& \ge \frac{\alpha}{2}  \norm{\Sigma_j ( \theta_j^* - \theta_t ) },
\end{align*}
where $\lambda_{\min}(\Sigma_j) \ge \alpha$ and  $\beta \ge \max_{j\in [k]}\norm{\Sigma_j}$.
It follows that 
$$
\left( \theta_j^* - \theta_{t+1} \right)^\top \Sigma^2_j \left( \theta_j^* - \theta_{t+1} \right)
\le \left( \theta_j^* - \theta_t \right)^\top \Sigma^2_j \left( \theta_j^* - \theta_t \right)
- \eta_t \alpha \norm{ \Sigma_j \left( \theta_j^* - \theta_t \right)} + \eta_t^2 \norm{\Sigma_j}^2.
$$
Recall the choice of step size
$
\eta_{t} =  \alpha \hat{\sigma}_{t} /(2\beta^2).
$
In view of~\prettyref{eq:sigma_t_bound}, we get that
\begin{align*}
\left( \theta_j^* - \theta_{t+1} \right)^\top \Sigma^2_j \left( \theta_j^* - \theta_{t+1} \right)
&\le \left( \theta_j^* - \theta_t \right)^\top \Sigma^2_j \left( \theta_j^* - \theta_t \right)
- \frac{\alpha^2 }{4 \beta^2 }\norm{ \Sigma_j \left( \theta_j^* - \theta_t \right)} \hat{\sigma}_t  \\
& \le \left( 1 - \frac{\alpha^2 }{8 \beta^2}  \right) \norm{ \Sigma_j \left( \theta_i^* - \theta_t \right)}^2,
\end{align*}

Therefore, 
 $$
 \norm{ \Sigma_j \left( \theta_j^* - \theta_{t+1} \right)}^2
 \le \left( 1 - \frac{\alpha^2 }{8 \beta^2}  \right) \norm{ \Sigma_j \left( \theta_j^* - \theta_{t} \right)}^2,
 $$
 This proves~\prettyref{eq:moment_decay}.


\section{Proofs for \prettyref{sec:iterative}}
Throughout the proof in this section, we assume without loss of generality that the optimal permutation in \eqref{eq:distance} is identity.

\subsection{Derivation of global iteration}
\label{app:global-iteration}
\begin{proof}[Proof of \prettyref{lmm:global-iteration}]
We first prove the result for FedAvg.
By definition, we have 
\[
\nabla_j L_i(\theta)
= \frac{\lambda_{ij,t}}{n_i}\ph{\bx_i}^\top (\ph{\bx_i}\theta_j-y_i),
\]
where $\lambda_{ij,t}=\indc{j=z_{i,t}}$ and $\nabla_j$ denotes the gradient with respect to $\theta_j$. 
Then the one-step local gradient descent at client $i$ is 
\[
[\calG_i(\theta)]_j=
\begin{cases}
\theta_j, & j\ne z_{i,t},\\
g_i(\theta_j)\triangleq \theta_j - \eta_i \ph{\bx_i}^\top(\ph{\bx_i}\theta-y_i),
& j=z_{i,t},
\end{cases}
\]
where $\eta_i=\eta/n_i$.
Iterating $s$ steps yields that \cite{su2021achieving}
\begin{align*}
g_i^s(\theta_j)
& =(I-\eta_i \ph{\bx_i}^\top \ph{\bx_i})^s\theta_j + \sum_{\ell=0}^{s-1}(I-\eta_i \ph{\bx_i}^\top \ph{\bx_i})^\ell \eta_i\ph{\bx_i}^\top y_i \\
& \overset{(a)}{=} \theta_j - \sum_{\ell=0}^{s-1}(I-\eta_i \ph{\bx_i}^\top \ph{\bx_i})^\ell \eta_i \ph{\bx_i}^\top (\ph{\bx_i} \theta_j -y_i) \\
& \overset{(b)}{=} \theta_j - \eta_i \ph{\bx_i}^\top P_i (\ph{\bx_i} \theta_j -y_i),
\end{align*}
where $(a)$ used $I-(I-X)^s=\sum_{\ell=0}^s(I-X)^\ell X$, and $(b)$ used $(I-X^\top X)^\ell X^\top = X^\top(I- X X^\top)^\ell$ and the definition of $P_i$.  
Then,
\begin{align*}
\theta_{ij,t}
& = [\calG_i^s(\theta_{t-1})]_j
= \lambda_{ij,t} g_i^s(\theta_{j,t-1})+ (1-\lambda_{ij,t}) \theta_{j,t-1}\\
& = \theta_{j,t-1} - \eta_i \lambda_{ij,t}\ph{\bx_i}^\top P_i (\ph{\bx_i} \theta_{j,t-1} -y_i). 
\end{align*}
We obtain the global iteration:
\[
\theta_{j,t}
= \sum_{i=1}^M \frac{n_i}{N}\theta_{ij,t}
= \theta_{j,t-1} - \frac{\eta}{N} \sum_{i=1}^M \lambda_{ij,t}\ph{\bx_i}^\top P_i (\ph{\bx_i} \theta_{j,t-1} -y_i),
\]
which is \eqref{eq:global-s} using matrix notations.

The proof for FedProx is similar. 
The first order condition for the local proximal optimization is 
\[
\eta_i \lambda_{ij,t}\ph{\bx_i}^\top (\ph{\bx_i}\theta_{ij,t}-y_i) + (\theta_{ij,t} - \theta_{j,t-1})=0,\quad j\in[k].
\]
Therefore, if $j\ne z_{i,t}$, then $\theta_{ij,t} = \theta_{j,t-1}$; if $j= z_{i,t}$, then
\begin{align*}
\theta_{ij,t} 
& = (I+\eta_i \ph{\bx_i}^\top \ph{\bx_i})^{-1}(\theta_{j,t-1}+\eta_i \ph{\bx_i}^\top y_i)\\
& \overset{(a)}{=} \theta_{j,t-1} - \eta_i (I+\eta_i \ph{\bx_i}^\top \ph{\bx_i})^{-1} \ph{\bx_i}^\top (\ph{\bx_i}\theta_{j,t-1}-y_i)\\
& \overset{(b)}{=} \theta_{j,t-1} - \eta_i \ph{\bx_i}^\top P_i (\ph{\bx_i}\theta_{j,t-1}-y_i),
\end{align*}
where $(a)$ used $I-(I+X)^{-1}=(I+X)^{-1}X$, and $(b)$ used $(I+X^\top X)^{-1}X^\top = X^\top(I+ X X^\top)^{-1}$ and the definition of $P_i$.
The remaining steps are the same as those in FedAvg.
\end{proof}

\subsection{Convergence analysis of \prettyref{phase:iterative-FL}}
\label{app:proof-iterative}
We analyze the three terms on the right-hand side of \eqref{eq:error-iteration} separately.
The first term of \eqref{eq:error-iteration} is the main term due to the decreasing of estimation error, and the last term is the stochastic variation due to the observation noise $\zeta$.
We have the following lemmas on the eigenvalues of $K_j$ and the concentration of the observation noise.

\begin{lemma}
\label{lmm:eigen-Sigma}
There exists constants $c$ and $C$ such that, with probability $1-2ke^{-d}$, 
\[
c\alpha\frac{s N_j}{\kappa N}
\le \lambda_{\min}(K_j)
\le \lambda_{\max}(K_j)
\le C\beta\frac{s N_j}{N},
\quad \forall j\in[k].
\]
\end{lemma}
\begin{proof}
Since $\ph{\bx_{I_j}}$ of size $N_j\times d$ consists of independent and sub-Gaussian rows, by a covering argument \cite[Theorem 4.6.1]{Vershynin2018}, with probability $1-2e^{-d}$,  
\[
\alpha N_j - C(\sqrt{dN_j}\vee d)
\le \sigma_{\min}^2(\ph{\bx_{I_j}})
\le \sigma_{\max}^2(\ph{\bx_{I_j}})
\le \beta N_j + C(\sqrt{dN_j}\vee d),
\]
where $\sigma_{\max}$ and $\sigma_{\min}$ denote the largest and smallest singular values, respectively, and $C$ is an absolute constant.
By definition, $K_j = \frac{1}{N}\ph{\bx_{I_j}}^\top P_{I_j} \ph{\bx_{I_j}}$, where $P_{I_j}$ is a symmetric matrix. 
It is shown in \cite[Lemma 3]{su2021achieving} that 
\[
s/\kappa
\le \lambda_{\min}(P_{I_j})
\le \lambda_{\max}(P_{I_j})
\le s. 
\]
The conclusion follows from the condition $N_j\gtrsim d$ and a union bound over $j\in[k]$. 
\end{proof}

\begin{lemma}
\label{lmm:noise}
Given the input features $\ph{\bx}$,
there exists a constant $C$ such that with probability at least 
$1-k\exp(-d)$,
\[
\Norm{B \Lambda_j \zeta}^2\le  C \frac{\sigma^2 s d }{N}\Norm{K_j},  \quad  \forall~j\in[k].
\]
\end{lemma}
\begin{proof}
Note that 
$$
\Norm{B \Lambda_j \zeta}^2 = \zeta^\top \Lambda_j B^\top B \Lambda_j \zeta
=\Iprod{\Lambda_j B^\top B\Lambda_j}{\zeta \zeta^\top}.
$$
Since $\expect{\zeta \zeta^\top} \preceq \sigma^2 I$, it follows that 
$$
\expect{ \Norm{B \Lambda_j \zeta}^2 } =
    \expect{\Iprod{\Lambda_j B^\top B\Lambda_j}{\zeta \zeta^\top}}
    \le \sigma^2 \Tr\left(\Lambda_j B^\top B\Lambda_j\right)
    = \sigma^2 \Tr\left(B \Lambda_j^2 B^\top \right).
$$
Recall that
\begin{align}
\label{eq:BBT-psd-ub}
B\Lambda_j^2B^\top 
& = \frac{1}{N^2} \ph{\bx_{I_j}}^\top P_{I_j}^2 \ph{\bx_{I_j}}
\overset{(a)}{\preceq} \frac{s}{N^2} \ph{\bx_{I_j}}^\top P_{I_j} \ph{\bx_{I_j}}
=\frac{s}{N}K_j,
\end{align}
where $(a)$ holds because $\Norm{P_{I_j}} \le s$.
Therefore, 
$$
\expect{ \Norm{B \Lambda_j \zeta}^2 } =\expect{\Iprod{\Lambda_j B^\top B\Lambda_j}{\zeta \zeta^\top}} 
\le \frac{\sigma^2 sd}{N}\norm{K_j}.
$$
Next, using Hanson-Wright's inequality~\cite{rudelson2013hanson}, we get 
\begin{align*}
&\phantom{{}={}}\prob{ \Iprod{\Lambda_j B^\top B\Lambda_j}{\zeta \zeta^\top} - 
\expect{\Iprod{\Lambda_j B^\top B\Lambda_j}{\zeta \zeta^\top}} \ge \delta} \nonumber \\
&\le \exp \left( -c_1 \min \left\{ \frac{\delta}{\sigma^2\Norm{\Lambda_j B^\top B\Lambda_j}}, \frac{\delta^2}{\sigma^4 \fnorm{\Lambda_j B^\top B\Lambda_j}^2} \right\} \right),  
\end{align*}
where $c_1>0$ is a universal constant. 
Note that 
\begin{align*}
\Norm{\Lambda_j B^\top B\Lambda_j}
& = \Norm{B \Lambda_j^2 B^\top}
\le \frac{s}{N}\Norm{K_j},\\
\fnorm{\Lambda_j B^\top B\Lambda_j}
& = \fnorm{B \Lambda_j^2 B^\top}
\le s \fnorm{K_j}
\le \frac{s \sqrt{d}}{N} \norm{K_j}.
\end{align*}
Therefore, 
by choosing $\delta = C \frac{\sigma^2 sd}{N}  \Norm{K_j} $ for a sufficiently large constant $C$, 
we get that with probability at least $1-\exp(-d)$,
$$
\Iprod{\Lambda_j B^\top B\Lambda_j}{\zeta \zeta^\top}
\le \expect{\Iprod{\Lambda_j B^\top B\Lambda_j}{\zeta \zeta^\top}} + \delta
\le \left(C+1\right)\sigma^2 \frac{s d }{N}\Norm{K_j}.
$$
The conclusion follows from a union bound over all $j\in[k]$.
\end{proof}

Combining Lemmas~\ref{lmm:clustering-error}, \ref{lmm:eigen-Sigma}, and \ref{lmm:noise}, next we prove \prettyref{thm:error-iteration}.
\begin{proof}[Proof of \prettyref{thm:error-iteration}]
We prove the result conditioning on the high probability events in  Lemmas~\ref{lmm:clustering-error}, \ref{lmm:eigen-Sigma}, and \ref{lmm:noise} that happen with probability at least $1-Cke^{-d}$.
We obtain from \prettyref{lmm:eigen-Sigma} that
\[
\Norm{I-\eta K_j}\le 1- C \eta s \rho/\kappa .
\]
Combining Lemmas \ref{lmm:eigen-Sigma} and \ref{lmm:noise} yields
\[
\Norm{B \Lambda_j \zeta}
\lesssim s \sigma \sqrt{\frac{d}{N}}. 
\]
Plugging the above upper bounds and \prettyref{lmm:clustering-error} into \eqref{eq:error-iteration}, we get
\[
\Norm{\theta_{j,t}-\theta_j^*}
\le \pth{1- C\eta s \pth{ \frac{\rho}{\kappa} - \nu\log\frac{e}{\nu}  } } d(\theta_{t-1},\theta^*) 
+ C\eta s \sigma \pth{\sqrt{\frac{d}{N}} + \nu\log\frac{e}{\nu}},
\quad \forall j\in[k].
\]
Since $\nu\log(e/\nu)\lesssim \rho/\kappa$ and $\nu\gtrsim \sqrt{d/N}$, we conclude \eqref{eq:fed_error_contraction}.

Let $\hat\theta_{i,t}=\theta_{z_{i,t},t}$  be  client $i$'s estimate of its own model parameter.
If client $i$ is clustered correctly such that $z_{i,t}=z_i$, where the success probability $\pprob{z_{i,t}=z_i}$ is shown in \prettyref{lmm:ub-prob-error} (which can be found in Appendix \ref{lmm:clustering-error}), it follows from \eqref{eq:fed_error_contraction} that, for $t\ge T+1$,
\[
\Norm{\hat\theta_{i,t}-\theta_{z_i}^*}
\le d(\theta_t,\theta^*)
\le \pth{1-C_1 s\eta \rho/\kappa}^{t-T} d(\theta_{T},\theta^*)
+ \frac{C_2}{C_1} \frac{\sigma\kappa}{\rho}   \nu\log\frac{e}{\nu}.
\]
The proof is completed.
\end{proof}

\subsubsection{Proof of \prettyref{lmm:clustering-error}}
This subsection is devoted to the proof of \prettyref{lmm:clustering-error} using the following road map: 
\[
d(\theta_t,\theta^*) 
\downarrow 
\implies
\sum_{i:i\in S_{j,t}} n_i 
\downarrow
\implies
\Norm{\ph{\bx_{S_{j,t}}}},\Norm{{\zeta_{S_{j,t}}}} 
\downarrow
\implies 
\Norm{B\calE_{j,t}(\ph{\bx}\theta_{j,t-1}-y)} 
\downarrow. 
\]
Specifically, a small estimation error $d(\theta_t,\theta^*)$ implies an upper bound on the total number of incorrectly clustered data points $\sum_{i\in S_{j,t}} n_i $; then we upper bound $\Norm{\ph{\bx_{S_j^t}}}$ and $\Norm{{\zeta_{S_j^t}}}$ using sub-Gaussian concentration and the union bound; finally we conclude the result from \eqref{eq:error-first-step}.

We first upper bound $\sum_{i\in S_{j,t}} n_i $. 
Using \eqref{eq:hat-z}, the set $S_{j,t}=I_j \ominus I_{j,t}$ is equivalently the union of
\begin{align*}
I_j-I_{j,t}
&=\sth{i\in I_j: \Norm{ y_i - \ph{\bx_i} \theta_{j,t-1}} \ge \min_{\ell\ne j} \Norm{ y_i - \ph{\bx_i} \theta_{\ell,t-1}} },
\\
I_{j,t}-I_j
&=\sth{i\not\in I_j: \Norm{ y_i - \ph{\bx_i} \theta_{j,t-1}} \le \min_{\ell\ne j} \Norm{ y_i - \ph{\bx_i} \theta_{\ell, t-1}} }.
\end{align*}
Therefore, $S_{j,t}=S_j(\theta_{t-1})$, where $S_j$ is defined in \eqref{eq:indc-error}.
The next lemma upper bounds the VC dimensions of the binary function classes specified in \eqref{eq:def-calF}. 

\begin{lemma}
\label{lmm:VC}
For $k\ge 2$, the VC dimensions of $\calF_j^{\RN{1}}$ and $\calF_j^{\RN{2}}$ are at most $O(dk\log k)$.
\end{lemma}
\begin{proof}
We focus on the proof for $\calF_{j}^{\RN 1}$ for a fixed $j\in[k]$, and the proof for $\calF_{j}^{\RN 2}$ is similar. 
We count the number of faces in the arrangement of geometric objects, which is also known as the number of sign patterns. 
Specifically, here we define the sign patterns of binary functions $g_1(\theta),\dots,g_m(\theta)$ as the set
\[
\sth{(g_1(\theta),\dots,g_m(\theta)):\theta\in\reals^{dk}}.
\]
Suppose $\calF_{j}^{\RN 1}$ shatters $m$ points denoted by $(\bx_1,y_1),\dots, (\bx_m,y_m)$.
Define binary functions 
\[
q_{i,\ell}(\theta)\triangleq \indc{P_{\ell,j}[\bx_i,y_i](\theta)\ge 0},
\quad 
g_i(\theta)\triangleq \max_{\ell\ne j} q_{i,\ell}(\theta).
\]
It is necessary that the number of sign patterns of $g_1(\theta),\dots,g_m(\theta)$ is $2^m$. 
Note that every $P_{\ell,j}[\bx_i,y_i]$ is a $(dk)$-variate quadratic function. 
By the Milnor-Thom theorem (see, e.g., \cite[Theorem 6.2.1]{Matousek2013}), the number of sign patterns of $m(k-1)$ binary functions $q_{1,\ell},\dots,q_{m,\ell}$ for $\ell\ne j$ is at most $(\frac{100m(k-1)}{dk})^{dk}$.
Since each $g_i$ is the maximum of $q_{i,\ell}$ over $\ell\ne j$, the number of sign patterns of $g_1,\dots,g_m$ is upper bounded by $(\frac{100m(k-1)}{dk})^{dk}$.
Consequently, we obtain $2^m\le (\frac{100m(k-1)}{dk})^{dk}$, and hence $m\lesssim dk\log k$.
\end{proof}

Next we show the uniform deviation of the incorrectly clustered data points. 
Due to the quantity skew, we consider a weighted empirical process $G_j(\theta)=\sum_{i=1}^M n_i\indc{i\in S_j(\theta)}$. 
Since the local data $(\bx_i,y_i)$ on different clients are independent, for a fixed $\theta$, the events $\{i\in S_j(\theta)\}$ as functions of $(\bx_i,y_i)$ are mutually independent. 
Using the binary function classes in \eqref{eq:def-calF}, we have
\begin{align}
&\phantom{{}={}}
\expect{\sup_{\theta}|G_j(\theta)-\Expect [G_j(\theta)] | }\nonumber\\
&\le \expect{\sup_{f\in \calF_j^{\RN{1}}}\abth{\sum_{i\in I_j}n_i (f(\bx_i,y_i)-\Expect[f(\bx_i,y_i)]) }  }
+\expect{\sup_{f\in \calF_j^{\RN{2}}}\abth{\sum_{i\not\in I_j}n_i (f(\bx_i,y_i)-\Expect[f(\bx_i,y_i)]) }  }\nonumber\\
& \lesssim \sqrt{dk\log k \sum_{i\in\calI_j}n_i^2  } 
+ \sqrt{dk\log k \sum_{i\not\in\calI_j}n_i^2  }\nonumber\\
& \le \sqrt{2dk\log k \sum_{i=1}^M n_i^2},\label{eq:expect-sup}
\end{align}
where the second inequality follows from the uniform deviation of weighted empirical processes in \prettyref{lmm:weighted-empirical} and the upper bound of VC dimensions in \prettyref{lmm:VC}.
Finally, we use the McDiarmid's inequality to establish a high-probability tail bound.
Note that we can write
$$
\sup_{\theta} \left| G_j(\theta)-\Expect [G_j(\theta)]  \right|
\triangleq h(Z_1, \ldots, Z_M)
$$
as a function $h$ of $Z_i=(\bx_i,y_i)$ with bounded differences:
for any $i, z_i, z_i',$
$$
\left| h(z_1, \ldots, z_i, \ldots, z_M)
-h(z_1, \ldots, z'_i, \ldots, z_M)
\right| \le n_i.
$$
By McDiarmid's inequality, we have
\begin{equation}
\label{eq:tail-McD}    
\prob{h(Z_1, \ldots, Z_M)
-\expect{h(Z_1, \ldots, Z_M)} 
\ge t} \le \exp \left( - \frac{2t^2}{\sum_{i=1}^M n_i^2} \right).
\end{equation}
Therefore, combining \eqref{eq:expect-sup} and \eqref{eq:tail-McD}, and by a union bound, with probablity at least $1-k^{-dk}$,
\begin{equation}
\label{eq:ub-uniform-deviation}
\sup_{\theta}|G_j(\theta)-\Expect [G_j(\theta)] |
\lesssim \sqrt{dk\log k \sum_{i=1}^M n_i^2}
= N\sqrt{\frac{dk\log k}{M}(\chi^2(n)+1)},\quad \forall j\in[k].
\end{equation}

\begin{lemma}
\label{lmm:ub-prob-error}
Suppose 
$\epsilon \le \frac{\sqrt{\alpha/\beta}}{3}\Delta$. 
Then,
\[
\sup_{\theta:d(\theta,\theta^*)\le \epsilon }
\Prob[i\in S_j(\theta)] \le 
4k\exp\pth{-cn_i\alpha^2\pth{1\wedge \frac{\Delta^2}{\sigma^2}}^2},\quad \forall j\in[k],
\]
where $c$ is an absolute constant.
\end{lemma}
\begin{proof}
For $i\in I_j$, it follows from \eqref{eq:indc-error} and the union bound that
\begin{align}
\prob{i\in S_j(\theta)}
&\le \sum_{\ell\ne j}\prob{\Norm{ y_i - \ph{\bx_i} \theta_j } \ge \Norm{ y_i - \ph{\bx_i} \theta_{\ell} } }\nonumber\\
&= \sum_{\ell\ne j}\prob{\Norm{ \ph{\bx_i} (\theta_j^*-\theta_j)+\zeta_i } \ge \Norm{ \ph{\bx_i} (\theta_j^*-\theta_\ell) +\zeta_i } }. \label{eq:prob-error-first-ub}
\end{align}
For any $u\in\reals^d$, the $n_i$-dimensional random vector $\ph{\bx_i}u+\zeta_i$ has independent and $(\Norm{u}^2+\sigma^2)$-sub-Gaussian coordinates. 
Applying Bernstein inequality yields that 
\begin{equation}
\label{eq:Bernstein}
\prob{\abth{ \frac{1}{n_i}\norm{\ph{\bx_i}u+\zeta_i}^2 - \pth{\Expect[\zeta_{i1}^2]+\lnorm{u}{\Sigma_i}^2} } \ge (\Norm{u}^2+\sigma^2)t}
\le 2\exp\pth{-cn_i (t\wedge t^2)},
\end{equation}
where $\Sigma_i=\Expect[\ph{x_{i1}}\ph{x_{i1}}^\top]$.
Let $u_1\triangleq \theta_j^*-\theta_j$ and $u_2\triangleq \theta_j^*-\theta_\ell$.
By assumptions that $\Norm{\theta_\ell-\theta_\ell^*}\le \epsilon$ for all $\ell\in[k]$ and $\Norm{\theta_j^*-\theta_\ell^*}\ge \Delta$ for $\ell\ne j$, 
we have $\Norm{u_1}\le \epsilon$, $\Norm{u_2}\ge \Delta-\epsilon$. 
Applying the condition $\epsilon\le \frac{1}{3\sqrt{\beta/\alpha}}\Delta \le \frac{1}{3}\Delta$, we get
\begin{equation}
\label{eq:test-seperation}
\lnorm{u_2}{\Sigma_i}^2 - \lnorm{u_1}{\Sigma_i}^2
\ge \alpha(\Delta-\epsilon)^2-\beta\epsilon^2
\ge \alpha \Delta^2/3.
\end{equation}
Therefore, let $m=\Expect[\zeta_{i1}^2]+(1-p)\lnorm{u_1}{\Sigma_i}^2 + p \lnorm{u_2}{\Sigma_i}^2$ with $p=\frac{\Norm{u_1}^2+\sigma^2}{\Norm{u_1}^2+\Norm{u_2}^2+2\sigma^2}$,
and we obtain from \eqref{eq:Bernstein} that
\begin{align*}
& \phantom{{}={}}\prob{\Norm{ \ph{\bx_i} (\theta_j^*-\theta_j)+\zeta_i } \ge \Norm{ \ph{\bx_i} (\theta_j^*-\theta_\ell) +\zeta_i } }\\
& \le \prob{\frac{1}{n_i}\Norm{\ph{\bx_i}u_1+\zeta_i}^2\ge m } 
+\prob{\frac{1}{n_i}\Norm{\ph{\bx_i}u_2+\zeta_i}^2\le m } \\
& \le 4\exp\pth{-cn_i (t\wedge t^2)},
\end{align*}
where $t=\frac{\lnorm{u_2}{\Sigma_i}^2 - \lnorm{u_1}{\Sigma_i}^2}{\Norm{u_1}^2+\Norm{u_2}^2+2\sigma^2}\gtrsim \alpha(1\wedge \frac{\Delta^2}{\sigma^2})$ using the lower bound of seperation in \eqref{eq:test-seperation}.
We conclude the proof for $i\in I_j$ from \eqref{eq:prob-error-first-ub}.
Similarly, for $i\in I_\ell$ with $\ell\ne j$, we have
\[
\prob{i\in S_j(\theta)}
\le \prob{\Norm{ \ph{\bx_i} (\theta_\ell^*-\theta_\ell)+\zeta_i } \ge \Norm{ \ph{\bx_i} (\theta_\ell^*-\theta_j) +\zeta_i } }. 
\]
The conclusion follows from a similar argument.
\end{proof}

Let $N_I\triangleq \sum_{i\in I}n_i$ denote the total number of data in a subset of clients $I\subseteq [M]$.
It follows from \eqref{eq:ub-uniform-deviation} and \prettyref{lmm:ub-prob-error} that, with probability $1-k^{-dk}$, 
\begin{equation}
\label{eq:clustering-error-N_S}
N_{S_{j,t}}=\sum_{i\in S_{j,t}}n_i
\le \nu N,
\end{equation}
where $\nu$ is defined in \eqref{eq:def_nu}.
Conditioning on total number of incorrectly clustered data points $N_{I}$, the next lemma upper bounds $\Norm{\ph{\bx_{I}}}$ and $\Norm{{\zeta_{I}}}$.

\begin{lemma}
\label{lmm:uniform-sub-Gaussian}
With probability $1-4e^{-d}$,  
\begin{equation}
\label{eq:uniform-sub-Gaussian}
\sup_{N_{I}\le \nu N}\frac{1}{N}\Norm{\ph{\bx_{I}}}^2
\lesssim \nu\log\frac{e}{\nu},
\qquad
\sup_{N_{I}\le \nu N}\frac{1}{N}\Norm{{\zeta_{I}}}^2
\lesssim \sigma^2 \nu\log\frac{e}{\nu}.
\end{equation}
\end{lemma}
\begin{proof}
Since $\ph{x_{ij}}$ are independent and 
sub-Gaussian random vectors in $\reals^d$,
for a fixed $I\subseteq[M]$, with probability at least $1-2e^{-t}$,
$$
\norm{\ph{\bx_I}}^2 
\le \beta N_I + C'\pth{\sqrt{(d+t)N_I} + (d+t) },
$$
for some absolute constant $C'>0$.
There are at most $\binom{N}{\nu N}\le \exp(N \nu \log (e/\nu))$ many different $I$ with $N_I \le N'$. 
Hence, applying the union bound yields that, with probability at least $1-2e^{-d}$,
\[
\sup_{N_{I}\le \nu N}\Norm{\ph{\bx_{I}}}^2
\lesssim N \nu\log\frac{e}{\nu},
\]
where we used $\nu \gtrsim \frac{d}{N}$.
Since $\zeta_{ij}$ are independent and sub-Gaussian with $\Expect[\zeta_{ij}^2]\le \sigma^2$, the second inequality in \eqref{eq:uniform-sub-Gaussian} follows from a similar argument. 
\end{proof}

Conditioning on the high probability events of \eqref{eq:clustering-error-N_S} and \eqref{eq:uniform-sub-Gaussian}, we obtain 
\[
\Norm{\ph{\bx_{S_{j,t}}}}\lesssim \sqrt{N \nu \log\frac{e}{\nu}},
\qquad
\Norm{\zeta_{S_{j,t}}}\lesssim \sigma\sqrt{N \nu \log\frac{e}{\nu}}.
\]
Since $\Norm{P_{S_{j,t}}}\le s$, 
we conclude from \eqref{eq:clustering-error} and \eqref{eq:error-first-step} that
\begin{align*}
\Norm{B\calE_{j,t}(\ph{\bx}\theta_{j,t-1}-y)}
& \le \frac{1}{N}\Norm{\ph{\bx_{S_{j,t}}}}\Norm{P_{S_{j,t}}}\Norm{\ph{\bx_{S_{j,t}}}\theta_{j,t-1}-y_{S_{j,t}}}\\
& \lesssim \frac{s}{N}\pth{d(\theta_{t-1},\theta^*)\Norm{\ph{\bx_{S_{j,t}}}}^2 + \Norm{\ph{\bx_{S_{j,t}}}}\Norm{\zeta_{S_{j,t}}} }\\
& \lesssim s (d(\theta_{t-1},\theta^*)+\sigma)\nu\log\frac{e}{\nu}.
\end{align*}

\subsubsection{Auxiliary lemma}
\label{subsubsec: auxiliary}
\begin{lemma}
\label{lmm:weighted-empirical}
Consider a weighted empirical process $G_n(f)=\sum_{i=1}^n \lambda_i f(X_i)$ for binary functions $f\in\calF$, where $X_i$'s are independent and the VC dimension of $\calF$ is at most $d$. 
Then
\[
\expect{\sup_{f\in\calF}|G_n(f)-\Expect G_n(f)|}
\lesssim \sqrt{d \sum_{i=1}^n \lambda_i^2 }.
\]
\end{lemma}
\begin{proof}
Since $X_i$'s are independent, by symmetrization,
\[
\expect{\sup_{f\in\calF}|G_n(f)-\Expect G_n(f)|}
\le 2\expect{\sup_{f\in\calF} \abth{\sum_{i=1}^n  \epsilon_i \lambda_i f(X_i)} },
\]
where $\epsilon_i$ are i.i.d.\ Rademacher random variables. 
Next, by conditioning on $X_i$'s, we aim to apply Dudley's integral.
Since $\epsilon_i$ are independent and 1-sub-Gaussian, for any $f, g \in \calF$, the increment $\sum_{i} \epsilon_i \lambda_i f(X_i) - \sum_{i} \epsilon_i \lambda_i g(X_i)$ is also sub-Gaussian with a variance parameter
\[
\sum_{i=1}^n \lambda_i^2 (f-g)(X_i)^2
=\pth{\sum_{i=1}^n \lambda_i^2} \lnorm{f-g}{L^2(\mu_n)}^2,
\]
where $\mu_n$ denotes the weighted empirical measure $\frac{1}{\sum_{i} \lambda_i^2} \sum_{i} \lambda_i^2 \delta_{X_i}.$
Apply Dudley's integral (see, e.g., \cite[Theorem 8.1.3]{Vershynin2018}) conditioning on $X_i$'s, we get that 
$$
\expect{\sup_{f\in\calF} 
\left| \sum_{i=1}^n \epsilon_i \lambda_i f(X_i) \right| }
\lesssim \sqrt{ \sum_{i=1}^n \lambda_i^2}
\times \expect{\int_0^1 \sqrt{ \log \calN(\calF, L^2(\mu_n), \epsilon) d\epsilon}},
$$
where $\calN(\calF,L^2(\mu_n),\epsilon)$ denotes the $\epsilon$-covering number of $\calF$ under $L^2(\mu_n)$. 
Finally, we can bound the covering number by the VC dimension of $\calF$ as (see, e.g., \cite[Theorem 8.3.18]{Vershynin2018})
$$
\log \calN(\calF, L^2(\mu_n), \epsilon)
\lesssim d \log \frac{2}{\epsilon}.
$$
The conclusion follows. 
\end{proof}

\section{Truncated matrix Bernstein inequality}

\begin{lemma}\label{lmm:matrix_berstein_1}
Let $X_i \in \reals^d$ be a sequence of independent sub-Gaussian random vectors. 
Let $\theta \in \reals^d$ be a fixed vector with unit $\ell_2$ norm.
Let $y_i= \iprod{X_i}{\theta} + z_i$, where $z_i$ is a sequence of independent sub-Gaussian random variables
that are independent from $X_i. $ 
Let $A_i= X_i X_i^\top y_i^2$.
Then with probability at least $1-3\delta$,
$$
\norm{\frac{1}{n}\sum_{i=1}^n \left( A_i - \expect{A_i} \right) } \le
C \left( \sqrt{nd \log \frac{1}{\delta}} + \left( d+ \log \frac{n}{\delta} \right) \log \frac{n}{\delta} \log \frac{1}{\delta} \right),
$$
where $C>0$ is some large constant.
\end{lemma}
\begin{proof}
For a given threshold $\tau$, we decompose
\begin{align*}
\norm{\frac{1}{n}\sum_{i=1}^n \left( A_i - \expect{A_i} \right)}
& \le \norm{\frac{1}{n}\sum_{i=1}^n \left( A_i \indc{\norm{A_i}\le \tau} - \expect{A_i\indc{\norm{A_i}\le \tau}} \right) } \\
& + \norm{\frac{1}{n} \sum_{i=1}^n A_i \indc{\norm{A_i}> \tau}}
+ \norm{ \frac{1}{n} \sum_{i=1}^n \expect{A_i\indc{\norm{A_i}> \tau}}}.
\end{align*}
In the sequel, we bound each term in the RHS separately. 

To bound the first term, we will use matrix Bernstein inequality. 
Let $B_i=A_i \indc{\norm{A_i}\le \tau} - \expect{A_i\indc{\norm{A_i}\le \tau}}$.
Note that $\norm{B_i} \le 2\tau$ 
and 
\begin{align*}
\sum_{i=1}^n \expect{B_i^2} & = 
\sum_{i=1}^n \expect{\left(A_i  \indc{\norm{A_i}\le \tau}- \expect{A_i  \indc{\norm{A_i}\le \tau}}\right)^2} \\
& =\sum_{i=1}^n
\left( \expect{A_i^2  \indc{\norm{A_i}\le \tau} } - \left( \expect{A_i \indc{\norm{A_i}\le \tau} } \right)^2 \right)
\end{align*}
Therefore, 
\begin{align*}
\sum_{i=1}^n \expect{B_i^2}   \preceq 
\sum_{i=1}^n 
 \expect{A_i^2 \indc{\norm{A_i}\le \tau}  }  \preceq \sum_{i=1}^n  \expect{A_i^2}. 
\end{align*}
Since $A_i=X_i X_i^\top y_i^2$ and $y_i=\iprod{X_i}{\theta} +z_i$, it follows that for any unit vector $u,$
\begin{align*}
u^\top \expect{A_i^2} u = \expect{ \iprod{X_i}{u}^2 \norm{X_i}^2 y_i^4} 
&\le \sqrt{\expect{\norm{X_i}^4}} \sqrt{\expect{\iprod{X_i}{u}^4 y_i^8} } \\
&\le \sqrt{\expect{\norm{X_i}^4}} \left(\expect{\iprod{X_i}{u}^8} \expect{y_i^{16}}\right)^{1/4}
= O(d)
\end{align*}
and hence $\norm{ \expect{A_i^2}} =O(d).$
Therefore, 
$
\norm{\sum_{i=1}^n \expect{B_i^2}}=O(nd).
$
Applying matrix Bernstein inequality~\cite{tropp2015introduction}, we get that with probability at least $1-\delta,$
$$
 \norm{ \sum_{i=1}^n B_i} \lesssim \sqrt{nd \log \frac{1}{\delta} } + \tau \log\frac{1}{\delta}.
$$

Next, we bound the second term.  Note that 
$$
\prob{ \frac{1}{n}  \sum_{i=1}^n A_i\indc{\norm{A_i}> \tau } =0} 
\ge \prob{ \norm{A_i} \le \tau, \forall i \in [n]}
\ge 1- \sum_{i=1}^n \prob{\norm{A_i} > \tau}. 
$$
Since $A_i=X_iX_i^\top y_i^2$, it follows that $\norm{A_i}=\norm{X_i}^2 y_i^2$. 
Since $X_i$ is a sub-Gaussian random vector, $\expect{\norm{X_i}^2}=O(d)$ and 
$$
\prob{\norm{X_i}^2 > \expect{\norm{X_i}^2} + C \left( \sqrt{d \log \frac{1}{\delta}} + \log \frac{1}{\delta} \right)} \le \delta.
$$
Similarly, $y_i$ is also sub-Gaussian. Therefore, $\expect{y_i^2}=O(1)$ and 
$$
\prob{\norm{y_i}^2 > \expect{\norm{y_i}^2} + C \log \frac{1}{\delta} } \le \delta.
$$
Combining the above two displayed equations gives that 
\begin{align}
\prob{\norm{A_i} >  C' \left( d+ \log \frac{1}{\delta} \right) \log \frac{1}{\delta}} \le 2\delta. \label{eq:norm_A_i}
\end{align}
where $C'$ is some universal constant. Thus by choosing $\tau=C' \left( d+ \log \frac{n}{\delta} \right) \log \frac{n}{\delta}$,
we get that
$\prob{\norm{A_i}  > \tau} \le 2\delta/n$ and 
hence $\frac{1}{n}  \sum_{i=1}^n A_i\indc{\norm{A_i}> \tau } =0$ with probability at least $1-2\delta.$

Finally, we bond the third term. Note that
\begin{align*}
    \norm{\sum_{i=1}^n \expect{A_i \indc{\norm{A_i} > \tau }}}
     \le \sum_{i=1}^n  \norm{ \expect{ A_i \indc{\norm{A_i} > \tau }}}
     \le \sum_{i=1}^N \expect{ \norm{A_i \indc{\norm{A_i} > \tau }}}.
\end{align*}
Moreover, 
\begin{align*}
 \expect{ \norm{A_i \indc{\norm{A_i} > \tau }}}
&=\int_{0}^\infty \prob{  \norm{A_i \indc{\norm{A_i} > \tau }}\ge t} \diff t  \\
&=\int_0^\tau \prob{  \norm{A_i } \ge \tau } \diff t + \int_\tau^\infty \prob{ \norm{A_i} \ge t } \diff t \\
&= \tau \frac{\delta}{n} +  \int_\tau^\infty \prob{ \norm{A_i} \ge t } \diff t.
\end{align*}
In view of~\prettyref{eq:norm_A_i}, we get that 
$
\prob{\norm{A_i} \ge C' d \log^2 \frac{1}{\delta} }\le 2 \delta,
$
or in other words, 
$
\prob{ \norm{A_i} \ge t } \le 2 \exp \left( - \sqrt{\frac{t}{C'd}}\right).
$
It follows that 
$$
\int_\tau^\infty \prob{ \norm{A_i} \ge t } \diff t
\le 2 \int_\tau^\infty  \exp \left( - \sqrt{\frac{t}{C'd}}\right) \diff t
=4C'd  \exp \left( - \sqrt{\frac{\tau}{C'd}}\right)\left( \sqrt{\frac{\tau}{C'd}}+1\right)=O(d).
$$
Combining the above bounds to the three terms, we arrive at the desired conclusion.

\end{proof}

\begin{lemma}\label{lmm:matrix_bernstein_2}
Let $\{ a_i: i \in [N]\}$ and $\{b_i: i \in [N]\}$ denote two independent sequences of independent random vectors in $\reals^{d}$. Suppose that 
$\expect{\norm{a_i}^2} =O(d)$, $\expect{\norm{b_i}^2} =O(d)$, $\norm{\expect{a_ia_i^\top}}=O(1)$, $\norm{\expect{b_ib_i^\top}}=O(1)$, and 
$$
\prob{ \norm{a_i} \ge t }, \prob{ \norm{b_i} \ge t } \le \exp\left( - \Omega\left(t /\sqrt{d} \right)  \right), \quad \forall t \ge \sqrt{d}. 
$$
Let 
$$
Y=\sum_{i=1}^N \left(a_i b_i^\top - \expect{a_i b_i^\top}\right).
$$
Then there exists a univeral constant $C>0$ such that with probability at least $1-3\delta,$
$$
\norm{Y} \le C \left(  \sqrt{Nd \log \frac{1}{\delta}} + d \log^3(N/\delta) \right).
$$
\end{lemma}
\begin{proof}
Given $\tau$ to be specified later, define event
$
\calE_i=\{\norm{a_ib_i^\top} \le \tau\}.
$
It follows that 
$$
Y= \sum_{i=1}^N \left(a_i b_i^\top \indc{\calE_i}- \expect{a_i b_i^\top \indc{\calE_i}}\right)
+ \sum_{i=1}^N a_i b_i^\top \indc{\calE_i^c} - \sum_{i=1}^N \expect{a_i b_i^\top \indc{\calE_i^c}}
$$
and hence
\begin{align}
\norm{Y} \le \norm{\sum_{i=1}^N \left(a_i b_i^\top \indc{\calE_i}- \expect{a_i b_i^\top \indc{\calE_i}}\right)}
+\norm{\sum_{i=1}^N a_i b_i^\top \indc{\calE_i^c} }
+\norm{\sum_{i=1}^N \expect{a_i b_i^\top \indc{\calE_i^c}}}. \label{eq:Y_norm_bound}
\end{align}
In the sequel, we bound each term in the RHS separately. 

To bound the first term, we will use matrix Bernstein inequality. 
Let $Y_i=a_i b_i^\top \indc{\calE_i}- \expect{a_i b_i^\top \indc{\calE_i}}$.
Then $\expect{Y_i}=0$ and 
$$
\norm{Y_i} \le \norm{a_i b_i^\top \indc{\calE_i}}+ \norm{\expect{a_i b_i^\top \indc{\calE_i}}} \le 2\tau.
$$
Moreover, 
\begin{align*}
\sum_{i=1}^N \expect{Y_iY_i^\top} & = 
\sum_{i=1}^N \expect{\left(a_i b_i^\top \indc{\calE_i}- \expect{a_i b_i^\top \indc{\calE_i}}\right) \left(a_i b_i^\top \indc{\calE_i}- \expect{a_i b_i^\top \indc{\calE_i}}\right)^\top} \\
& =\sum_{i=1}^N 
\left( \expect{a_i a_i^\top \norm{b_i}^2 \indc{\calE_i} } -\expect{a_i b_i^\top \indc{\calE_i}}\expect{a_i b_i^\top \indc{\calE_i}}^\top \right)
\end{align*}
Therefore, 
\begin{align*}
\sum_{i=1}^N \expect{Y_iY_i^\top}  & \preceq 
\sum_{i=1}^N 
 \expect{a_i a_i^\top \norm{b_i}^2 \indc{\calE_i} } \\
 & \preceq \sum_{i=1}^N 
 \expect{a_i a_i^\top \norm{b_i}^2} \\
 &= \sum_{i=1}^N \expect{\norm{b_i}^2} \expect{a_ia_i^\top} \preceq O(Nd) \identity.
\end{align*}
Moreover, $Y_iY_i^\top \succeq 0$. Hence,
$
\norm{\sum_{i=1}^N \expect{Y_iY_i^\top}}=O(Nd).
$
Similarly, we can show that 
$
\norm{\sum_{i=1}^N \expect{Y_i^\top Y_i}}=O(Nd).
$
Applying matrix Bernstein inequality~\cite{tropp2015introduction}, we get that with probability at least $1-\delta,$
\begin{align}
 \norm{ \sum_{i=1}^N Y_i} \lesssim \sqrt{Nd \log \frac{1}{\delta} } + \tau \log\frac{1}{\delta}.
\label{eq:Y_norm_bound_1}
\end{align}

Next, we bound the second term in~\prettyref{eq:Y_norm_bound}. Note that on the event
$\cap_{i=1}^N \calE_i$, $\norm{\sum_{i=1}^N a_i b_i^\top \indc{\calE_i^c} }=0.$ 
Note that 
$$
\prob{\calE_i^c} = \prob{ \norm{a_i b_i^\top} > \tau }
\le \prob{ \norm{a_i} \ge \sqrt{\tau}} + \prob{ \norm{b_i} \ge \sqrt{\tau}}
\le 2 e^{-\Omega(\sqrt{\tau/d})}.
$$
Hence by choosing $\tau=C d \log^2 \frac{N}{\delta}$ for some sufficiently large constant $C$,
we get that $\prob{\calE_i^c} \le \delta/N$. Thus by union bound, 
\begin{align}
 \prob{\cap_{i=1}^N \calE_i} \ge 1- \sum_{i=1}^N \prob{\calE_i^c} \ge 1- 2\delta.
   \label{eq:Y_norm_bound_2}
\end{align}


Finally, we bond the third term in~\prettyref{eq:Y_norm_bound}.
Note that
\begin{align*}
    \norm{\sum_{i=1}^N \expect{a_i b_i^\top \indc{\calE_i^c}}}
     \le \sum_{i=1}^N  \norm{\expect{a_i b_i^\top \indc{\calE_i^c}}}
     \le \sum_{i=1}^N \expect{ \norm{a_i b_i^\top \indc{\calE_i^c}}}.
\end{align*}
Moreover, 
\begin{align*}
 \expect{ \norm{a_i b_i^\top \indc{\calE_i^c}}}
&=\int_{0}^\infty \prob{ \norm{a_i b_i^\top \indc{\calE_i^c}} \ge t} \diff t  \\
&=\int_0^\tau \prob{ \norm{a_i b_i^\top } \ge \tau } \diff t + \int_\tau^\infty \prob{ \norm{a_i b_i^\top } \ge t } \diff t \\
&= \tau \frac{\delta}{N} + \int_\tau^\infty \prob{ \norm{a_i b_i^\top } \ge t } \diff t
\end{align*}
By assumption, for $t \ge \tau = Cd \log^2 \frac{N}{\delta}$,
$$
\prob{ \norm{a_i b_i^\top } \ge t }
\le \prob{ \norm{a_i} \ge \sqrt{t} } + \prob{ \norm{b_i} \ge \sqrt{t} } 
\le 2 e^{-C'\sqrt{t/d}}
$$
for some universal constant $C'>0$.
It follows that 
$$
\int_\tau^\infty \prob{ \norm{a_i b_i^\top } \ge t } \diff t
\le 2 \int_\tau^\infty e^{-C'\sqrt{t/d}} \diff t
=4d \left(  \sqrt{\tau/d} + 1/C' \right) e^{-C' \sqrt{\tau/d} },
$$
where the equality holds by the identity that $\int_{\tau}^\infty e^{-\alpha \sqrt{t}} \diff t =\frac{2}{\alpha^2} ( \sqrt{\tau} \alpha +1 ) e^{-\alpha \sqrt{\tau}}. $
Therefore,
\begin{align}
 \expect{ \norm{a_i b_i^\top \indc{\calE_i^c}}} 
 \le \tau \frac{\delta}{N} + 4d \left(  \sqrt{\tau/d} + 1/C' \right) e^{-C' \sqrt{\tau/d} }
 = O\left( \frac{d}{N} \log^2 (N/\delta) \right). \label{eq:Y_norm_bound_3}
\end{align}
Plugging~\prettyref{eq:Y_norm_bound_1}, \prettyref{eq:Y_norm_bound_2}, and~\prettyref{eq:Y_norm_bound_3} into~\prettyref{eq:Y_norm_bound} yields the desired conclusion. 
\end{proof}

\section{Bound on the largest principal angle between random subspaces}
Let $U \in \reals^{d \times \ell}$ denote an orthogonal matrix and $Q \in \reals^{d \times \ell}$ denote a random orthogonal matrix chosen uniformly at random, where $\ell \le d.$

\begin{lemma}\label{lmm:subspace_angle}
With probability at least $1-O(\epsilon),$
$$
\sigma_{\min}(U^\top Q) \gtrsim \frac{\epsilon}{\sqrt{\ell}(\sqrt{d}+\log (1/\epsilon))}. 
$$
\end{lemma}
\begin{proof}
Since $Q \in \reals^{d\times \ell}$ is a random orthogonal matrix, to prove the claim, without loss of generality, we can assume $U=[e_1, e_2, \ldots, e_\ell]$, where $e_i$'s are the standard basis vectors in $\reals^d$. 
Let $A \in \reals^{d\times \ell}$ denote a random Gaussian matrix with i.i.d.\ $\calN(0,1)$ entries and
write
$
A=
\begin{bmatrix}
X \\
Y
\end{bmatrix},
$
where $X \in \reals^{\ell \times \ell}$ and $Y\in \reals^{(d-\ell)\times \ell}$.
Then $U^\top Q$ has the same distribution as $X(A^\top A)^{-1/2}$.
It follows that $\sigma_{\min}(U^\top Q)$ has the same distribution as $\sigma_{\min}(X(A^\top A)^{-1/2})$. Note that
$$
\sigma_{\min}\left(X(A^\top A)^{-1/2}\right)
\ge \sigma_{\min} (X) \sigma_{\min}\left( (A^\top A)^{-1/2}\right)
=\frac{\sigma_{\min} (X)}{\sigma_{\max}(A)}.
$$
In view of~\cite[Corollary 5.35]{vershynin2010introduction}, $\sigma_{\max} (A)
\lesssim \sqrt{d} + \log (1/\epsilon)$ with probability at least $1-\epsilon.$
Moreover, in view of~\cite[Theorem 1.2]{szarek1991condition},
$\sigma_{\min}(X) \ge \epsilon/\sqrt{\ell}$ with probability at least $1-O(\epsilon)$. 
The desired conclusion readily follows.
\end{proof}

\section*{Acknowledgement}
J.~Xu is supported by the NSF Grants IIS-1838124, CCF-1850743, CCF-1856424,
and CCF-2144593.
P.~Yang is supported by the NSFC Grant 12101353.

\bibliography{reference,main}
\bibliographystyle{alpha}

\end{document}